\documentclass[11pt, a4paper]{article}

\usepackage[T1]{fontenc}
\usepackage{amsmath, amssymb, amsthm}
\usepackage{geometry}
\usepackage{natbib}
\usepackage{mathtools} 
\usepackage{xcolor}    
\usepackage{setspace}  
\usepackage[colorlinks=true, linkcolor=blue, citecolor=blue, urlcolor=blue]{hyperref}
\usepackage{appendix}  
\usepackage{graphicx}  
\usepackage{caption}   
\usepackage{subcaption} 
\usepackage{booktabs}  
\usepackage{algorithm} 
\usepackage{algpseudocode} 
\usepackage{filecontents} 

\newtheorem{theorem}{Theorem}[section]
\newtheorem{corollary}[theorem]{Corollary}

\newtheorem{proposition}[theorem]{Proposition}
\newtheorem{definition}[theorem]{Definition}
\theoremstyle{remark}
\newtheorem{remark}[theorem]{Remark}

\newtheorem{assumption}[theorem]{Assumption}

\newcommand{\R}{\mathbb{R}}
\newcommand{\E}{\mathbb{E}}
\newcommand{\Tr}{\mathrm{Tr}}
\newcommand{\U}{\mathcal{U}}
\newcommand{\A}{\mathcal{A}}
\newcommand{\norm}[1]{\left\lVert#1\right\rVert}
\newcommand{\Sspace}{\mathcal{S}^n} 
\newcommand{\nablaV}{\nabla V} 
\newcommand{\HessV}{D^2V} 
\newcommand{\Deltae}{\Delta_\epsilon} 
\newcommand{\Aeff}{A_{\mathrm{eff},0}} 
\newcommand{\Llin}{\mathcal{L}_{\mathrm{lin}}} 

\hypersetup{
    pdftitle={Robust Control with Gradient Uncertainty},
    pdfauthor={Anonymous},
    pdfsubject={Operations Research},
    pdfkeywords={Robust Control, Dynamic Programming, Viscosity Solutions, Reinforcement Learning, Knightian Uncertainty}
}

\begin{document}

\begin{titlepage}

\title{Robust Control with Gradient Uncertainty}

\author{Qian Qi\thanks{Peking University, Beijing 100871, China. Email: \href{mailto:qiqian@pku.edu.cn}{qiqian@pku.edu.cn}}}
\date{} 

\maketitle
\vfill

\begin{abstract}
We introduce a novel extension to robust control theory that explicitly addresses uncertainty in the value function's gradient, a form of uncertainty endemic to applications like reinforcement learning where value functions are approximated. We formulate a zero-sum dynamic game where an adversary perturbs both system dynamics and the value function gradient, leading to a new, highly nonlinear partial differential equation: the Hamilton-Jacobi-Bellman-Isaacs Equation with Gradient Uncertainty (GU-HJBI). We establish its well-posedness by proving a comparison principle for its viscosity solutions under a uniform ellipticity condition. Our analysis of the linear-quadratic (LQ) case yields a key insight: we prove that the classical quadratic value function assumption fails for any non-zero gradient uncertainty, fundamentally altering the problem structure. A formal perturbation analysis characterizes the non-polynomial correction to the value function and the resulting nonlinearity of the optimal control law, which we validate with numerical studies. Finally, we bridge theory to practice by proposing a novel Gradient-Uncertainty-Robust Actor-Critic (GURAC) algorithm, accompanied by an empirical study demonstrating its effectiveness in stabilizing training. This work provides a new direction for robust control, holding significant implications for fields where function approximation is common, including reinforcement learning and computational finance.
\end{abstract}

\vfill

\noindent\textbf{Subject Classifications:} \\
\textbf{Dynamic programming/optimal control:} robust control under state-dependent ambiguity. \\
\textbf{Stochastic models:} viscosity solutions for second-order nonlinear PDEs. \\
\textbf{Games/group decisions, stochastic:} zero-sum dynamic games with value function ambiguity. \\
\textbf{Reinforcement learning:} robust algorithms, actor-critic methods.

\end{titlepage}

\newpage

\section{Introduction}
\label{sec:introduction}

The theory of stochastic optimal control provides a powerful mathematical framework for decision-making under uncertainty. At its heart lies the value function, which quantifies the optimal expected future cost from any given state. The behavior of this central object is described by the celebrated Hamilton-Jacobi-Bellman (HJB) equation, a foundation of dynamic programming (e.g., \cite{Bellman1957}).

In many real-world scenarios, the assumption of a perfectly known model for the system dynamics is untenable. Robust control theory addresses this challenge by reformulating the control problem as a zero-sum game between the controller and an adversarial "Nature." This approach, pioneered in a modern context by \cite{hansen2001robust,HansenSargent2008}, assumes the adversary perturbs the system dynamics to maximize the agent's cost, while the agent seeks a policy that is robust to this worst-case scenario. This game-theoretic perspective leads to the Hamilton-Jacobi-Bellman-Isaacs (HJBI) equation, which incorporates a penalty for model misspecification. A key feature of this standard framework is that the adversary's optimal perturbation is directly proportional to the gradient of the value function, $\nablaV(x)$. The gradient, representing the marginal cost of a state change, is thus the very instrument the adversary uses to inflict maximal damage. This implicitly assumes that both the agent and the adversary know this gradient with perfect precision.

However, in a vast array of modern applications, this assumption is questionable. In reinforcement learning (RL), for instance, the value function is rarely known in closed form and is instead approximated from data using function approximators like neural networks (e.g., \cite{mnih2015human,SuttonBarto2018}). The gradient of this approximated value function is therefore inherently uncertain. Similarly, in mathematical finance, the sensitivities of an option's price to market parameters (the "Greeks") are derived from models that are themselves imperfect approximations of reality (e.g., \cite{Cont2006}). This observation motivates the central question of this paper:

\textit{How should a controller act when they are uncertain not only about the model dynamics but also about the marginal value of their own state?}

To answer this question, we propose a new framework for robust control that explicitly incorporates this second layer of uncertainty. We model the agent's ambiguity about its value function gradient by allowing the adversary to choose a pointwise perturbation to the gradient from within a prescribed uncertainty set. This introduces a novel adversarial component into the dynamic game, leading to a highly nonlinear HJBI-type equation that features a complex coupling between the gradient, the control, and the diffusion process.

\subsection{Related Literature and Contributions}

Our work builds upon several rich traditions in control theory and operations research.

\paragraph{Robust Control.} The formulation of control as a zero-sum differential game has deep roots, originating with the work of \cite{Isaacs1965}. In the modern era, the classic approach to robustness in continuous-time control is $H_\infty$ control (see \cite{BasarOlsder1999}), which seeks to minimize the worst-case gain from an external disturbance to a system output. The framework of \cite{HansenSargent2008} provides a powerful stochastic interpretation that unifies these ideas, connecting the robustness penalty to the relative entropy between a nominal and a worst-case model. This notion of robustness against an ambiguous model is rooted in the decision-theoretic concept of Knightian uncertainty and max-min expected utility (see \cite{GilboaSchmeidler1989,chen2002ambiguity}). In discrete time, this has led to the extensive field of robust Markov Decision Processes (MDPs), where uncertainty is typically modeled as residing in the transition probabilities or rewards (e.g., \cite{NilimElGhaoui2005, Iyengar2005}). Our work differs from this entire body of literature by placing the uncertainty directly on the agent's internal state valuation (the gradient), rather than on the external model parameters.

\paragraph{Viscosity Solutions.} The value functions in stochastic control and differential games are often non-differentiable, necessitating a generalized notion of solution for the associated PDEs. The theory of viscosity solutions, introduced by \cite{CrandallLions1983}, provides the unique mathematical framework for making sense of Hamilton-Jacobi equations. Seminal works such as \cite{CrandallLions1992} and early applications to Isaacs equations from differential games by \cite{EvansSouganidis1984} established viscosity solutions as the canonical tool. The definitive text for the application of this theory to stochastic control is \cite{FlemingSoner2006}, whose methods we will heavily rely upon to establish the well-posedness of our new GU-HJBI equation.

\paragraph{Value Function Approximation in RL.} The motivation for our work is strongly tied to the practical realities of reinforcement learning. In methods like Q-learning or actor-critic, an agent learns a value function or policy from sampled data (e.g., \cite{SuttonBarto2018}). When function approximators are used (see \cite{qian2025icml,qi2025universalappr,qi2025neuralhamiltonianoperator}), the resulting value function estimates are inherently noisy, and their gradients can be unreliable, leading to well-known stability issues (e.g., \cite{Baird1995}). Actor-critic methods, in particular, rely on these gradients to update the policy (e.g., \cite{KondaTsitsiklis2000}). Modern deep RL algorithms like DDPG (see \cite{Lillicrap2015}) are susceptible to this issue. Some research has sought to stabilize training by implicitly managing gradient quality through trust regions or clipping (see \cite{Schulman2015, Schulman2017}). The field of robust RL has emerged to address these challenges more directly, but has largely focused on robustness to external model misspecification or adversarial attacks on state observations (see \cite{Pinto2017}). Our work addresses a more fundamental source of error: the agent's imperfect self-knowledge of its own value function gradient .

\paragraph{Main Contributions.}
This paper makes the following principal contributions:
\begin{enumerate}
    \item \textbf{Formulation:} We formulate a novel robust control problem that incorporates ambiguity over the value function's gradient, leading to a new class of dynamic programming equations: the HJBI Equation with Gradient Uncertainty (GU-HJBI).
    \item \textbf{Well-Posedness:} We rigorously establish the well-posedness of the GU-HJBI equation. We provide a detailed proof of the comparison principle for viscosity solutions and prove the existence of a solution using Perron's method.
    \item \textbf{LQ Analysis:} We conduct a detailed analysis of the linear-quadratic (LQ) case. We prove that the classical quadratic value function ansatz fails for any non-zero gradient uncertainty.
    \item \textbf{Perturbation Analysis and Nonlinearity:} Through a formal perturbation expansion up to second order, we characterize the corrections to the value function and optimal control. We show this correction is the solution to a linear PDE with a non-polynomial source term, which in turn renders the optimal control law nonlinear.
    \item \textbf{Numerical Validation:} We provide extensive numerical examples, including a sensitivity analysis and a 2D problem, that solve the perturbation equations, visually demonstrating the non-quadratic nature of the value function and the nonlinearity of the optimal control.
    \item \textbf{Conceptual Analysis:} We analyze how different geometries for the gradient uncertainty set lead to different robust penalties and clarify how our framework extends beyond the standard relative entropy interpretation of robustness.
    \item \textbf{A Bridge to RL:} We propose a concrete, novel algorithm, the Gradient-Uncertainty-Robust Actor-Critic (GURAC), which translates our theoretical framework into a practical tool for training more robust RL agents, and provide empirical validation of its effectiveness.
\end{enumerate}

The paper is organized as follows. Section \ref{sec:problem_formulation} reviews the standard stochastic control problem and the classical robust control framework. Section \ref{sec:hjbi_gu} introduces our novel problem formulation and derives the GU-HJBI equation. Section \ref{sec:well_posedness} is dedicated to the theoretical analysis of this new PDE, establishing a comparison principle and existence. Section \ref{sec:lq_analysis} provides a deep dive into the linear-quadratic case, proving the failure of the quadratic ansatz and performing a detailed perturbation analysis. Section \ref{sec:numerical_studies} presents numerical studies that illustrate our theoretical findings. Section \ref{sec:geometry_and_entropy} discusses the implications of different uncertainty geometries. Section \ref{sec:rl_connection} forges the connection to reinforcement learning and proposes our new algorithm and its empirical validation. Section \ref{sec:conclusion} concludes and outlines future research. Proofs are in the Appendices.

\section{Problem Formulation and Background}
\label{sec:problem_formulation}

We begin by establishing the mathematical setting. Let $(\Omega, \mathcal{F}, \{\mathcal{F}_t\}_{t \ge 0}, \mathbb{P})$ be a complete filtered probability space, where the filtration $\{\mathcal{F}_t\}_{t \ge 0}$ is the $\mathbb{P}$-augmentation of the natural filtration generated by a standard $m$-dimensional Brownian motion $\{B_t\}_{t \geq 0}$.

\subsection{The Standard Stochastic Control Problem}

We consider a controlled stochastic process $X_t$ in $\R^n$ governed by the stochastic differential equation (SDE):
\begin{equation} \label{eq:sde}
    dX_t = f(t, X_t, u_t) dt + \sigma(t, X_t, u_t) dB_t, \quad X_0 = x,
\end{equation}
where $u_t$ is a control process taking values in a compact set of admissible controls $\U \subset \R^k$. A control policy (or strategy) is a process $u = \{u_t\}_{t \ge 0}$ that is progressively measurable with respect to the filtration $\{\mathcal{F}_t\}_{t \ge 0}$. We denote the set of all such admissible policies by $\A$.

The controller's objective is to choose a policy $u \in \A$ to minimize a cost functional over an infinite horizon:
\begin{equation}
    J(x; u) \coloneqq \E_x \left[ \int_0^\infty e^{-\rho t} L(t, X_t, u_t) dt \right],
\end{equation}
where $L: \R \times \R^n \times \U \to \R$ is a running cost function, $\rho > 0$ is a constant discount factor, and $\E_x[\cdot]$ denotes the expectation conditional on $X_0=x$. For simplicity, we will focus on the time-homogeneous case where $f$, $\sigma$, and $L$ do not depend explicitly on time $t$. Throughout this paper, we impose the following standard assumptions on the problem data.
\begin{assumption}[Standard Assumptions] \label{ass:standard}
The functions $f: \R^n \times \U \to \R^n$, $\sigma: \R^n \times \U \to \R^{n \times m}$, and $L: \R^n \times \U \to \R$ satisfy:
\begin{itemize}
    \item[(A1)] \textbf{Continuity:} They are continuous in all their arguments.
    \item[(A2)] \textbf{Lipschitz Continuity:} They are Lipschitz continuous with respect to the state variable $x$, uniformly in the control $u \in \U$. That is, there exists a constant $K_L > 0$ such that for all $x_1, x_2 \in \R^n$ and $u \in \U$:
    \begin{align*}
        \norm{f(x_1, u) - f(x_2, u)} + \norm{\sigma(x_1, u) - \sigma(x_2, u)}_{F} + |L(x_1, u) - L(x_2, u)| \le K_L \norm{x_1 - x_2},
    \end{align*}
    where $\norm{\cdot}_F$ is the Frobenius norm.
    \item[(A3)] \textbf{Linear Growth:} They satisfy a linear growth condition with respect to $x$, uniformly in $u \in \U$. That is, there exists a constant $K_G > 0$ such that for all $x \in \R^n$ and $u \in \U$:
    \begin{align*}
        \norm{f(x, u)}^2 + \norm{\sigma(x, u)}_{F}^2 + |L(x, u)| \le K_G(1 + \norm{x}^2).
    \end{align*}
\end{itemize}
\end{assumption}
These assumptions ensure that for any given control policy $u \in \A$, the SDE \eqref{eq:sde} has a unique strong solution, and the cost functional is well-defined. The value function for this optimal control problem is defined as:
\begin{equation}
    V(x) \coloneqq \inf_{u \in \A} J(x; u).
\end{equation}
Under our assumptions, the value function is continuous and is the unique viscosity solution to the Hamilton-Jacobi-Bellman (HJB) equation:
\begin{equation} \label{eq:hjb}
    \rho V(x) = \inf_{u \in \U} \left\{ L(x,u) + \mathcal{L}^u V(x) \right\},
\end{equation}
where $\mathcal{L}^u$ is the second-order infinitesimal generator of the process $X_t$ under control $u$:
\[
    \mathcal{L}^u \phi(x) \coloneqq \nabla\phi(x)^T f(x,u) + \frac{1}{2} \Tr\left(\sigma(x,u)\sigma(x,u)^T \HessV\phi(x)\right).
\]
Here, $\nabla\phi$ denotes the gradient of $\phi$ and $\HessV\phi$ denotes its Hessian matrix.

\subsection{The Standard Robust Control Framework}

The theory of robust control extends this problem by considering a malevolent adversary who perturbs the dynamics. A common framework, related to risk-sensitive control, introduces a perturbation $h_t \in \R^m$ that enters the drift via the diffusion channel:
\begin{equation} \label{eq:sde_perturbed}
    dX_t = (f(X_t, u_t) + \sigma(X_t, u_t)h_t) dt + \sigma(X_t, u_t) dB_t.
\end{equation}
The term $\sigma(X_t, u_t)h_t$ represents a misspecification of the drift dynamics. This specific structure is crucial, as by Girsanov's theorem, this change in drift is equivalent to a change of probability measure. The adversary is penalized for the size of the perturbation, which corresponds to the relative entropy (or Kullback-Leibler divergence) between the original and the perturbed measure. This leads to a zero-sum game with the robust value function:
\begin{equation}
    V(x) = \inf_{u \in \A} \sup_{h \in \mathcal{H}} \E_x \left[ \int_0^\infty e^{-\rho t} \left( L(X_t, u_t) - \frac{1}{2\eta} \norm{h_t}^2 \right) dt \right],
\end{equation}
where $\mathcal{H}$ is the set of admissible perturbation processes and the parameter $\eta > 0$ models the agent's level of concern about model misspecification (a larger $\eta$ implies more concern). The term $\frac{1}{2\eta}\norm{h_t}^2$ serves as a running cost for the adversary.

The dynamic programming principle for this game leads to the Hamilton-Jacobi-Bellman-Isaacs (HJBI) equation. We can derive it heuristically by first solving the inner maximization problem at a point $(x, u)$:
\[
\sup_{h \in \R^m} \left\{ \nablaV(x)^T \sigma(x,u)h - \frac{1}{2\eta} \norm{h}^2 \right\}.
\]
This is a simple concave maximization problem. The first-order condition yields the optimal attack:
\begin{equation}
    h^*(x,u) \coloneqq \eta \, \sigma(x,u)^T \nablaV(x).
\end{equation}
Substituting this back gives the maximized value $\frac{\eta}{2}\norm{\sigma(x,u)^T \nablaV(x)}^2$. Inserting this into the HJB equation gives the standard HJBI equation for robust control:
\begin{equation} \label{eq:hjbi_standard}
\begin{split}
    \rho V(x) = \inf_{u \in \U} \bigg\{ & L(x,u) + \nablaV(x)^T f(x,u) + \frac{1}{2}\Tr\left(\sigma(x,u)\sigma(x,u)^T \HessV (x)\right) \\
    & + \frac{\eta}{2}\norm{\sigma(x,u)^T \nablaV(x)}^2 \bigg\}.
\end{split}
\end{equation}
The crucial term $\frac{\eta}{2}\norm{\sigma^T \nablaV}^2$ is the penalty for robustness. It reveals that the adversary's optimal attack is precisely aligned with the system's sensitivity to noise, weighted by the value function's gradient.

\section{The HJBI Equation with Gradient Uncertainty}
\label{sec:hjbi_gu}

The standard robust framework assumes that $\nablaV(x)$ is known perfectly. We now relax this assumption. We suggest that the controller is concerned that the true sensitivity of their value function to model perturbations is not $\nablaV(x)$, but rather $\nablaV(x) + \delta$, where $\delta$ is an adversarial perturbation chosen pointwise in space to maximize the controller's cost. This captures the agent's ambiguity about the local marginal value of states, an ambiguity that is highly relevant in settings where $V$ is learned or approximated.

\begin{definition}[Uncertainty Set for the Gradient]
Let $\epsilon \geq 0$ be a parameter representing the magnitude of gradient uncertainty. For any state $x \in \R^n$, the set of admissible gradient perturbations is the closed ball
\begin{equation}
    \Deltae \coloneqq \{ \delta \in \R^n \mid \norm{\delta} \leq \epsilon \}.
\end{equation}
We primarily use the Euclidean norm ($\ell_2$), but we will discuss other geometries in Section \ref{sec:geometry_and_entropy}.
\end{definition}

The adversary chooses $\delta$ with full knowledge of the current state $x$ and the agent's nominal gradient $\nablaV(x)$. This $\delta$ is thus a state-dependent function, $\delta(x)$, representing a pointwise, localized attack on the agent's valuation.

The agent now plays a game against an adversary who controls both the drift perturbation $h$ and the gradient perturbation $\delta$. The agent seeks a value function $V$ by solving the following robust optimization problem, expressed in its dynamic programming form:
\begin{equation} \label{eq:new_pde}
\begin{split}
    \rho V(x) = \inf_{u \in \U} \sup_{h \in \R^m, \delta \in \Deltae} \bigg\{ & L(x,u) + (\nablaV(x) + \delta)^T (f(x,u) + \sigma(x,u)h) \\
    & - \frac{1}{2\eta} \norm{h}^2 + \frac{1}{2} \Tr\left(\sigma(x,u)\sigma(x,u)^T \HessV(x)\right) \bigg\}.
\end{split}
\end{equation}
This is the central PDE of our paper, which we call the \textbf{Hamilton-Jacobi-Bellman-Isaacs Equation with Gradient Uncertainty (GU-HJBI)}. For a fixed state $x$, control $u$, gradient $p \coloneqq \nablaV(x)$, and gradient perturbation $\delta$, the inner maximization with respect to $h$ is strictly concave. This allows us to first solve for the optimal drift perturbation.

\begin{proposition}[Reduced GU-HJBI Equation] \label{prop:reduced_pde}
The GU-HJBI equation \eqref{eq:new_pde} is equivalent to the following equation, where the maximization over $h$ has been resolved:
\begin{equation} \label{eq:hjbi_gu_reduced}
\begin{split}
    \rho V(x) = \inf_{u \in \U} \bigg\{ & L(x,u) + \frac{1}{2} \Tr\left(\sigma\sigma^T \HessV(x)\right) \\
    & + \sup_{\delta \in \Deltae} \left[ (\nablaV(x) + \delta)^T f(x,u) + \frac{\eta}{2} \norm{\sigma(x,u)^T(\nablaV(x) + \delta)}^2 \right] \bigg\}.
\end{split}
\end{equation}
\end{proposition}
\begin{proof}
The proof is a straightforward completion of the square and is provided in Appendix \ref{app:proof_reduced_pde} for completeness. The optimal drift perturbation is $h^*(x, u, p, \delta) = \eta \sigma(x,u)^T (p + \delta)$, where $p=\nablaV(x)$.
\end{proof}

The remaining maximization over $\delta$ is the novel and most challenging part of this equation. Let $p = \nablaV(x)$ and $S(x,u) = \sigma(x,u)\sigma(x,u)^T$. The inner problem is to maximize the convex quadratic function $\Phi(\delta) \coloneqq (p+\delta)^T f(x,u) + \frac{\eta}{2}(p+\delta)^T S(x,u) (p+\delta)$ over the compact ball $\norm{\delta} \leq \epsilon$. To understand the impact of this new term, it is instructive to perform a first-order expansion for small $\epsilon$.

\begin{proposition}[Hamiltonian Expansion for Small Gradient Uncertainty] \label{prop:hamiltonian_expansion}
Let $p = \nablaV(x)$. Define the robust Hamiltonian component corresponding to the gradient uncertainty as
\[
\mathcal{G}(x, u, p) \coloneqq \sup_{\norm{\delta} \leq \epsilon} \left\{ (p+\delta)^T f(x,u) + \frac{\eta}{2} \norm{\sigma(x,u)^T(p+\delta)}^2 \right\}.
\]
For small $\epsilon > 0$, this function has the expansion:
\begin{equation} \label{eq:G_expansion}
\mathcal{G}(x,u,p) = \left(p^T f(x,u) + \frac{\eta}{2}\norm{\sigma(x,u)^T p}^2\right) + \epsilon\norm{f(x,u) + \eta\sigma(x,u)\sigma(x,u)^T p} + O(\epsilon^2).
\end{equation}
\end{proposition}
\begin{proof}
The proof is provided in Appendix \ref{app:proof_hamiltonian_expansion}. It relies on identifying the linear term in an expansion of the objective in $\delta$ and maximizing it over the ball $\norm{\delta} \le \epsilon$.
\end{proof}

This expansion is highly revealing. The first term is simply the standard robust control term. The first-order correction due to gradient uncertainty is $\epsilon\norm{f(x,u) + \eta\sigma\sigma^T \nablaV(x)}$. The vector $v(x,u,p) \coloneqq f + \eta\sigma\sigma^T p$ can be interpreted as the drift sensitivity of the agent's objective to the gradient perturbation. The agent is penalized by the Euclidean norm of this vector. This term is nonlinear and, crucially, non-differentiable with respect to its vector argument at the origin. This suggests that for small $\epsilon$, the GU-HJBI equation can be approximated by:
\begin{equation} \label{eq:hjbi_approx}
\begin{split}
    \rho V(x) \approx \inf_{u \in \U} \bigg\{ & L(x,u) + \nablaV(x)^T f(x,u) + \frac{1}{2}\Tr\left(\sigma\sigma^T \HessV(x)\right) \\
    & + \frac{\eta}{2}\norm{\sigma(x,u)^T \nablaV(x)}^2 + \epsilon\norm{f(x,u) + \eta\sigma(x,u)\sigma(x,u)^T \nablaV(x)} \bigg\}.
\end{split}
\end{equation}
The presence of the non-differentiable norm term suggests that the value function itself may lose smoothness, reinforcing the need for the viscosity solution framework.

\section{Well-Posedness of the GU-HJBI Equation}
\label{sec:well_posedness}

To ensure that the GU-HJBI equation \eqref{eq:hjbi_gu_reduced} has a unique, meaningful solution, we analyze it within the framework of viscosity solutions. We first define the Hamiltonian and then state and prove the main comparison principle, followed by an existence result.

\subsection{Hamiltonian and Viscosity Solution Definition}
Following standard practice in viscosity solution theory, we write the GU-HJBI equation \eqref{eq:hjbi_gu_reduced} in the form $F(x, V, \nablaV, \HessV) = 0$. Let the Hamiltonian for a fixed control $u$ be
\begin{align*}
\mathcal{H}(x, p, X, u) \coloneqq & L(x,u) + \frac{1}{2}\Tr\left(\sigma(x,u)\sigma(x,u)^T X\right) \\
& + \sup_{\norm{\delta}\le\epsilon} \left[ (p+\delta)^T f(x,u) + \frac{\eta}{2}\norm{\sigma(x,u)^T(p+\delta)}^2 \right].
\end{align*}
The full PDE can then be expressed as:
\begin{equation} \label{eq:hamiltonian_def}
F(x, r, p, X) \coloneqq \rho r - \inf_{u \in \U} \mathcal{H}(x, p, X, u) = 0.
\end{equation}
Under Assumption \ref{ass:standard}, the functions $L, f, \sigma$ are continuous, and the optimizations over compact sets ensure that the resulting function is continuous in its arguments $(x, r, p, X)$. The key structural property is its monotonicity with respect to the Hessian matrix $X$. The Hamiltonian is proper or degenerate elliptic, meaning that $F(x, r, p, X) \ge F(x, r, p, Y)$ whenever $X \ge Y$ in the sense of symmetric matrices. This holds because $\sigma\sigma^T$ is positive semidefinite, so for any $u$, $\Tr(\sigma\sigma^T(X-Y)) \ge 0$, which carries through the infimum.

We now formally define viscosity solutions for our equation. Let USC (LSC) denote the space of upper (lower) semicontinuous functions.

\begin{definition}[Viscosity Solution]
Let $V: \R^n \to \R$ be a locally bounded function.
\begin{enumerate}
    \item $V$ is a \textbf{viscosity subsolution} of $F(x, V, \nablaV, \HessV)=0$ if for any test function $\phi \in C^2(\R^n)$ and any point $x_0$ which is a local maximum of $V-\phi$, we have
    \[ F(x_0, V(x_0), \nabla\phi(x_0), D^2\phi(x_0)) \le 0. \]
    \item $V$ is a \textbf{viscosity supersolution} of $F(x, V, \nablaV, \HessV)=0$ if for any test function $\phi \in C^2(\R^n)$ and any point $x_0$ which is a local minimum of $V-\phi$, we have
    \[ F(x_0, V(x_0), \nabla\phi(x_0), D^2\phi(x_0)) \ge 0. \]
    \item $V$ is a \textbf{viscosity solution} if it is both a subsolution and a supersolution.
\end{enumerate}
\end{definition}

\subsection{Comparison Principle and Uniqueness}
The cornerstone for proving uniqueness of the viscosity solution is a comparison principle. It asserts that any subsolution must lie below any supersolution, implying that there can be at most one continuous solution. To establish this, we require a stronger condition on the diffusion term.

\begin{assumption}[Uniform Ellipticity] \label{ass:ellipticity}
There exists a constant $\nu > 0$ such that for all $(x,u) \in \R^n \times \U$, the diffusion matrix is uniformly positive definite:
\[ \sigma(x,u)\sigma(x,u)^T \ge \nu I, \]
where $I$ is the $n \times n$ identity matrix.
\end{assumption}

\begin{theorem}[Comparison Principle] \label{thm:comparison}
Let Assumptions \ref{ass:standard} and \ref{ass:ellipticity} hold. Let $u \in \text{USC}(\R^n)$ be a bounded viscosity subsolution and $v \in \text{LSC}(\R^n)$ be a bounded viscosity supersolution of the GU-HJBI equation. Then $u(x) \leq v(x)$ for all $x \in \R^n$.
\end{theorem}
\begin{proof}
The proof employs the standard doubling of variables method, a cornerstone technique in viscosity solution theory. The detailed derivation, which adapts the classical proof to our specific Hamiltonian structure, is provided in Appendix \ref{app:proof_comparison}. The uniform ellipticity assumption is critical to control the Hessian terms and ensure the stability of the argument.
\end{proof}

\begin{corollary}[Uniqueness]
Under the assumptions of Theorem \ref{thm:comparison}, there exists at most one bounded continuous viscosity solution to the GU-HJBI equation.
\end{corollary}

\subsection{Existence of a Solution}
Uniqueness is powerful, but only if a solution is known to exist. The existence of a viscosity solution is typically established using Perron's method. This involves defining a candidate solution as the supremum over all subsolutions that satisfy a certain growth condition and then showing that this candidate is, in fact, a solution.

\begin{theorem}[Existence of a Solution]
Under Assumptions \ref{ass:standard} and \ref{ass:ellipticity}, there exists a unique bounded and continuous viscosity solution to the GU-HJBI equation \eqref{eq:hjbi_gu_reduced}.
\end{theorem}
\begin{proof} (Sketch)
The proof follows from Perron's method \cite[see][Chapter 4]{FlemingSoner2006}. The key ingredients are:
\begin{enumerate}
    \item \textbf{Comparison Principle:} This was established in Theorem \ref{thm:comparison}.
    \item \textbf{Existence of Sub- and Supersolutions:} We must show that the set of subsolutions is not empty. One can typically find constants $C_1, C_2$ such that the constant function $\psi_1(x) = -C_1$ is a subsolution and $\psi_2(x) = C_2$ is a supersolution. This follows from the boundedness of the coefficients for bounded $x$.
    \item \textbf{Stability:} The Hamiltonian must be stable under taking suprema. That is, if we define the function $V(x) = \sup \{ \psi(x) \mid \psi \text{ is a subsolution and } \psi \le \psi_2 \}$, we need to show that $V$ is itself a viscosity solution.
\end{enumerate}
The comparison principle ensures that $V$ is well-defined and that $V \le \psi_2$. Showing $V$ is a supersolution is relatively direct. Showing it is a subsolution is more involved and relies on the full machinery of viscosity solution theory, including the construction of local test functions. Given these standard (but technical) steps, the existence of a unique continuous solution is guaranteed. It can also be shown that the value function defined by the underlying game is this unique solution.
\end{proof}

\begin{remark}[On the Role of Uniform Ellipticity]
The uniform ellipticity assumption (Assumption \ref{ass:ellipticity}) is crucial for our proof of the comparison principle. It ensures the Hamiltonian is strictly monotonic in the Hessian variable $X$, which is essential for the doubling of variables proof to succeed. Relaxing this to the degenerate case ($\nu=0$) is a highly technical challenge. The difficulty arises from the novel term $\norm{\sigma^T(p+\delta)}^2$, which creates a complex coupling between the adversary's choices and the system dynamics. An adversary could specifically choose a gradient perturbation $\delta$ to exploit directions of degeneracy in the diffusion matrix $\sigma$. This interaction might break the structural coercivity properties of the Hamiltonian that are relied upon in standard techniques for degenerate PDEs. We leave this significant technical extension to future work.
\end{remark}

\section{Analysis of the Linear-Quadratic Case}
\label{sec:lq_analysis}

To gain deeper insight into the effects of gradient uncertainty, we now specialize our analysis to the canonical linear-quadratic (LQ) setting. Here, the dynamics are linear and the costs are quadratic.
\begin{align*}
    dX_t &= (Ax_t + Bu_t)dt + \Sigma dB_t, \\
    L(x,u) &= x^T Q x + u^T R u,
\end{align*}
where $A \in \R^{n \times n}$, $B \in \R^{n \times k}$, $\Sigma \in \R^{n \times m}$ are constant matrices. We assume $Q \in \Sspace$ is symmetric and positive semidefinite ($Q \ge 0$) and $R \in \mathcal{S}^k$ is symmetric and positive definite ($R > 0$). The control set is $\U = \R^k$.

\begin{assumption}[Stabilizability and Detectability] \label{ass:lq_stabilizability}
For the LQ problem, we assume that the pair $(A,B)$ is stabilizable and the pair $(A, Q^{1/2})$ is detectable.
\end{assumption}
These standard assumptions ensure the existence of a unique stabilizing solution to the relevant algebraic Riccati equation.

\subsection{Failure of the Quadratic Ansatz}
In the classical robust LQ framework ($\epsilon=0$), it is a celebrated result that the value function is a quadratic function of the state, $V(x) = x^T P x + c$. The matrix $P$ is found by solving an Algebraic Riccati Equation (ARE). We now show that this fundamental property is destroyed by the introduction of gradient uncertainty.

\begin{proposition} \label{prop:failure_quadratic}
For any $\epsilon > 0$ and any non-degenerate problem data, the value function $V(x)$ of the LQ problem with gradient uncertainty is not a quadratic function of the state $x$.
\end{proposition}
\begin{proof}
We proceed by contradiction. Assume the value function has the quadratic form $V(x) = x^T P x + c$ for some symmetric matrix $P \in \Sspace$ and constant $c \in \R$. The gradient is $\nablaV(x) = 2Px$ and the Hessian is $\HessV(x) = 2P$. Substituting this ansatz into the full GU-HJBI equation \eqref{eq:hjbi_gu_reduced} is complex. However, using the first-order expansion from Proposition \ref{prop:hamiltonian_expansion} is sufficient to reveal the structural mismatch. Substituting the quadratic ansatz into the approximate GU-HJBI equation \eqref{eq:hjbi_approx} yields:
\begin{equation} \label{eq:ansatz_pde_lq_approx}
\begin{split}
    \rho (x^T P x + c) = \inf_{u \in \R^k} \bigg\{ & x^T Q x + u^T R u + (2Px)^T(Ax+Bu) + \frac{1}{2}\Tr(\Sigma\Sigma^T(2P)) \\
    & + \frac{\eta}{2}\norm{\Sigma^T (2Px)}^2 + \epsilon\norm{Ax+Bu + 2\eta\Sigma\Sigma^T Px} \bigg\}.
\end{split}
\end{equation}
The optimal control $u$ that minimizes the sum of quadratic terms is linear in $x$: $u^*(x) = -R^{-1}B^T P x$. Let us substitute this into the equation. The left-hand side (LHS) is a quadratic polynomial in $x$. The right-hand side (RHS) consists of several terms. All terms are quadratic in $x$ except for the last one:
\[ \epsilon\norm{Ax - B(R^{-1}B^T P x) + 2\eta\Sigma\Sigma^T Px} = \epsilon \norm{(A - BR^{-1}B^TP + 2\eta\Sigma\Sigma^T P)x}. \]
This term is of the form $\epsilon \norm{M x}$ for a constant matrix $M$. For the equality to hold for all $x \in \R^n$, the functional form of both sides of the equation must match. A quadratic polynomial (the LHS) cannot be identically equal to the sum of a quadratic polynomial and a non-polynomial term (the RHS) over the entire space $\R^n$, unless the non-polynomial term is zero. Since $\norm{Mx}$ is not a quadratic polynomial for any non-zero matrix $M$, this equality cannot hold unless $\epsilon = 0$. This contradicts the assumption that $\epsilon > 0$. Therefore, the initial ansatz that $V(x)$ is quadratic must be false.
\end{proof}
\begin{remark}
The argument above, based on the first-order expansion in $\epsilon$, is sufficient to prove the result. The conclusion holds even more strongly for the full GU-HJBI equation, as the term $\sup_{\norm{\delta}\le\epsilon} \Phi(\delta)$ is generally not a quadratic function of $x$ (via its dependence on $p=2Px$), further reinforcing the structural mismatch.
\end{remark}

This result is profound. It implies that even for the simplest class of control problems, gradient uncertainty introduces a fundamental nonlinearity that cannot be handled by classical Riccati-based methods.

\subsection{Perturbation Analysis for Small Epsilon}
Since a closed-form solution is unavailable, we turn to perturbation analysis to understand the structure of the solution for small $\epsilon > 0$. We suggest an expansion for the value function and the optimal control:
\begin{align*}
V(x) &= V_0(x) + \epsilon V_1(x) + \epsilon^2 V_2(x) + O(\epsilon^3) \\
u^*(x) &= u_0(x) + \epsilon u_1(x) + \epsilon^2 u_2(x) + O(\epsilon^3)
\end{align*}
We substitute these expansions into the approximate GU-HJBI equation \eqref{eq:hjbi_approx} and collect terms of like powers in $\epsilon$. The rigor of this formal procedure can be established using the Implicit Function Theorem in appropriate function spaces, as sketched in Appendix \ref{app:perturb_rigor}.

\paragraph{Zeroth-Order Problem (Terms of order $\epsilon^0$).}
Setting $\epsilon=0$ recovers the standard robust LQ problem. The equation for $V_0$ is:
\[
\rho V_0 = \inf_{u} \left\{ x^TQx + u^TRu + \nablaV_0^T(Ax+Bu) + \frac{\eta}{2}\norm{\Sigma^T \nablaV_0}^2 + \frac{1}{2}\Tr(\Sigma\Sigma^T \HessV_0) \right\}.
\]
As is standard, we set the ansatz $V_0(x) = x^T P_0 x + c_0$. Then $\nablaV_0(x) = 2P_0x$ and $\HessV_0(x) = 2P_0$. The optimal control is found from the first-order condition on $u$:
\[ 2Ru + B^T \nablaV_0 = 0 \implies u_0(x) \coloneqq -R^{-1}B^T P_0 x. \]
Substituting $V_0$ and $u_0$ back in and equating coefficients of the quadratic forms in $x$ yields the celebrated \textbf{Robust Algebraic Riccati Equation (ARE)} for the symmetric matrix $P_0$:
\begin{equation} \label{eq:are}
    A^T P_0 + P_0 A + Q - P_0 B R^{-1} B^T P_0 + 2\eta P_0 \Sigma\Sigma^T P_0 = \rho P_0.
\end{equation}
The constant term yields $c_0 = \frac{1}{\rho}\Tr(\Sigma\Sigma^T P_0)$. Under Assumption \ref{ass:lq_stabilizability}, this ARE has a unique symmetric positive definite solution $P_0$ that results in a stable closed-loop system (e.g., \cite{BasarOlsder1999}).
The zeroth-order closed-loop dynamics matrix is stable and is given by
\begin{equation} \label{eq:acl}
    A_{cl} \coloneqq A - B R^{-1} B^T P_0.
\end{equation}
The full  drift for the zeroth-order problem, including the robust term, is $\Aeff \coloneqq A_{cl} + 2\eta\Sigma\Sigma^T P_0$.

\paragraph{First-Order Problem (Terms of order $\epsilon^1$).}
We now collect terms of order $\epsilon^1$. This involves linearizing the HJB operator around the zeroth-order solution $(V_0, u_0)$ and adding the explicit $\epsilon$-order source term from the gradient uncertainty penalty. The resulting equation for the first-order correction $V_1$ is a linear PDE.

The source term, $H_1(x)$, is the coefficient of $\epsilon$ in the expansion of the full Hamiltonian, evaluated at the zeroth-order solution ($p=\nablaV_0, u=u_0$):
\begin{align*}
H_1(x) &\coloneqq \norm{f(x,u_0(x)) + \eta\Sigma\Sigma^T \nablaV_0(x)} \\
&= \norm{Ax + B(-R^{-1}B^T P_0 x) + 2\eta\Sigma\Sigma^T (P_0 x)} \\
&= \norm{(A - BR^{-1}B^T P_0 + 2\eta\Sigma\Sigma^T P_0)x} = \norm{\Aeff x}.
\end{align*}
The final PDE for the first-order correction $V_1(x)$ is a Lyapunov-like equation:
\begin{equation} \label{eq:v1_pde}
    \rho V_1(x) = (\nabla V_1(x))^T (\Aeff x) + \frac{1}{2}\Tr(\Sigma\Sigma^T \HessV_1(x)) + \norm{\Aeff x}.
\end{equation}
This is a second-order linear PDE for $V_1(x)$. Its generator corresponds to an Ornstein-Uhlenbeck process with the stable drift matrix $\Aeff$. The Feynman-Kac formula gives its solution as an expectation:
\[ V_1(x) = \E_x \left[ \int_0^\infty e^{-\rho t} \norm{\Aeff Z_t} dt \right], \quad \text{where } dZ_t = \Aeff Z_t dt + \Sigma dB_t, \quad Z_0 = x. \]
The crucial observation is that the source term, $\norm{\Aeff x}$, is a norm of a linear function of $x$, which is generally not a polynomial. Consequently, the solution $V_1(x)$ to this linear PDE will, in general, be a non-polynomial function of $x$.

The first-order correction to the optimal control law is found by linearizing the optimality condition for $u$:
\[ u_1(x) \coloneqq -R^{-1}B^T \nabla V_1(x). \]
Since $V_1(x)$ is non-polynomial, its gradient $\nabla V_1(x)$ will be a nonlinear function of $x$. This means the optimal control law becomes nonlinear, a direct consequence of the agent's uncertainty.

\paragraph{Second-Order Problem (Terms of order $\epsilon^2$).}
Collecting terms of order $\epsilon^2$ reveals further complexity. The equation for $V_2(x)$ will take the form:
\begin{equation} \label{eq:v2_pde}
    \rho V_2(x) = \Llin[V_2](x) + H_2(x),
\end{equation}
where $\Llin[W](x) = (\nabla W(x))^T (\Aeff x) + \frac{1}{2}\Tr(\Sigma\Sigma^T \HessV W(x))$ is the same linear operator as before. The source term $H_2(x)$ is more complex and arises from several sources detailed in Appendix \ref{app:second_order}. In summary, $H_2(x)$ includes non-polynomial functions of $x$ and its gradient, such as terms proportional to $u_1^T R u_1 = (\nabla V_1)^T B R^{-1} B^T \nabla V_1$ and terms from the second-order expansion of the norm penalty. The key takeaway is that the complexity introduced by the non-polynomial nature of $V_1$ propagates and magnifies at higher orders. Solving for $V_2$ would require solving another linear PDE, but with an even more complicated source term derived from the solution for $V_1$.

\section{Numerical Analysis}
\label{sec:numerical_studies}

To make our theoretical findings concrete and to validate the predictions of the perturbation analysis, we present a series of numerical examples. We first analyze the 1D case in detail, including a sensitivity analysis of the key uncertainty parameters. We then present a 2D case to illustrate the richer geometric structure of the solution in higher dimensions. For all examples, the linear PDEs for the perturbation correction terms are solved using standard numerical methods (finite differences or finite elements), as detailed in Appendix \ref{app:numerical_methods}.

\subsection{One-Dimensional LQ Problem}

\subsubsection{Problem Setup}
We begin with a scalar system ($n=k=m=1$) governed by:
\begin{align*}
    dX_t &= (ax_t + bu_t)dt + \sigma dB_t, \quad L(x,u) = qx^2 + ru^2
\end{align*}
We choose the following baseline parameters, where $a=0.5$ renders the open-loop system unstable, creating a non-trivial control problem.
\begin{center}
\begin{tabular}{llc}
\toprule
\textbf{Parameter} & \textbf{Description} & \textbf{Value} \\
\midrule
$a$ & System drift & 0.5 \\
$b$ & Control effectiveness & 1.0 \\
$\sigma$ & Volatility & 1.0 \\
$q$ & State cost & 1.0 \\
$r$ & Control cost & 1.0 \\
$\rho$ & Discount factor & 0.1 \\
$\eta$ & Model uncertainty parameter & 0.2 \\
\bottomrule
\end{tabular}
\end{center}

\subsubsection{Zeroth and First-Order Solutions}
The Robust ARE \eqref{eq:are} becomes a scalar quadratic equation: $2ap_0 + q - (b^2/r)p_0^2 + 2\eta\sigma^2 p_0^2 = \rho p_0$. For the baseline parameters, this yields $0.6p_0^2 - 0.9p_0 - 1 = 0$. The unique positive, stabilizing root is $p_0 \approx 2.264$. This gives a zeroth-order linear control law $u_0(x) = -(b/r)p_0 x = -2.264x$.

The effective drift for the $V_1$ PDE is $a_{\mathrm{eff},0} = a - (b^2/r)p_0 + 2\eta\sigma^2 p_0 \approx -0.8584$. The PDE for the first-order correction $V_1(x)$ is therefore:
\begin{equation} \label{eq:v1_pde_1d}
    \rho V_1(x) = a_{\mathrm{eff},0} x V_1'(x) + \frac{1}{2}\sigma^2 V_1''(x) + |a_{\mathrm{eff},0}x|.
\end{equation}
We solve this second-order linear ODE numerically using a finite difference scheme on a bounded domain.

\subsubsection{Results and Interpretation}
Figure \ref{fig:1d_results} presents the numerically computed value function and optimal control law for a gradient uncertainty level of $\epsilon=0.5$. The results provide a clear visual confirmation of our primary theoretical predictions.

The left panel shows that the total value function, approximated as $V_0(x) + \epsilon V_1(x)$ (blue solid line), visibly deviates from the purely quadratic base solution $V_0(x)$ (red dashed line). This departure from the quadratic form, predicted in Proposition \ref{prop:failure_quadratic}, is a direct result of the non-polynomial correction term $V_1(x)$. The V-shape of the underlying source term $|a_{\mathrm{eff},0}x|$ induces a non-quadratic component in the solution, which is most pronounced for states away from the origin.

The right panel illustrates the consequence for the control policy. The total optimal control law, $u^*(x) \approx u_0(x) + \epsilon u_1(x)$, is clearly nonlinear. The correction $u_1(x) = -(b/r)V_1'(x)$ is derived from the gradient of the non-polynomial function $V_1(x)$, thus breaking the linearity of the classical LQ regulator. The control becomes more aggressive (steeper) than the linear law for states far from the origin, which is precisely where the drift sensitivity and thus the potential impact of gradient uncertainty are largest. The small, high-frequency oscillations are numerical artifacts from the finite-difference gradient computation and do not alter the fundamental nonlinear character of the solution.

\begin{figure}[ht!]
    \centering
    \includegraphics[width=\linewidth]{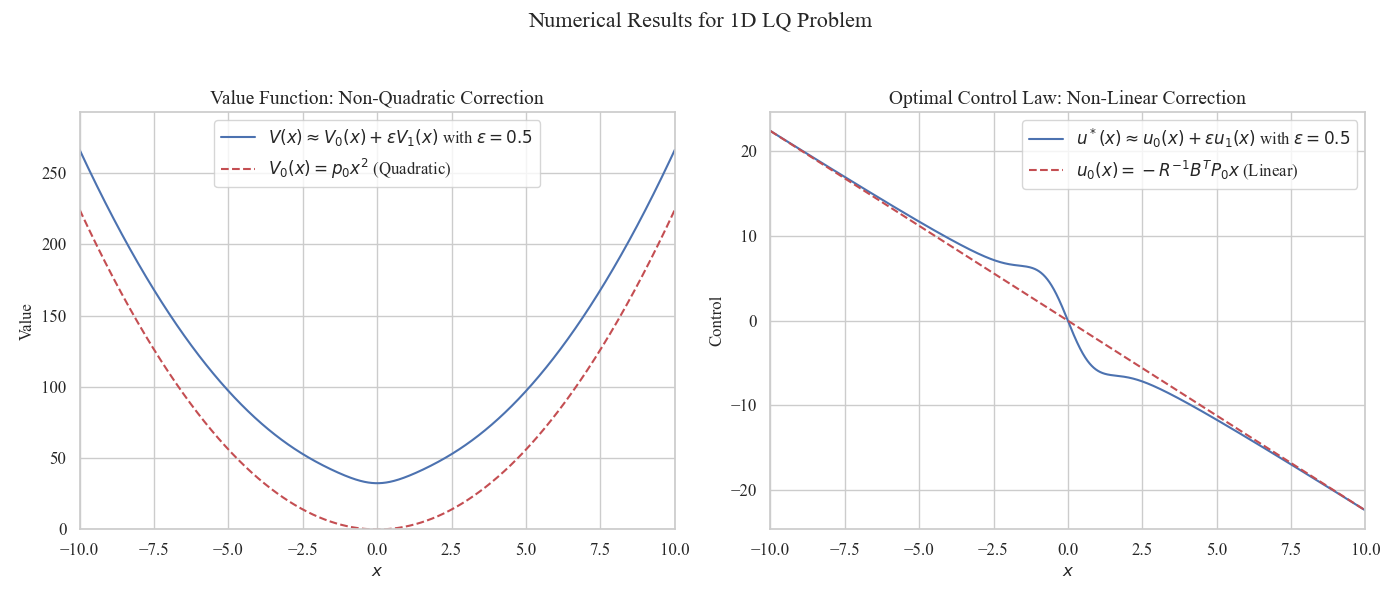} 
    \caption{Numerically computed value function and control law for the 1D LQ problem with $\epsilon=0.5$. (Left) The total value function (blue) visibly deviates from the purely quadratic $V_0$ (red dashed), confirming the failure of the quadratic ansatz. (Right) The optimal control law (blue) is nonlinear, in contrast to the linear law $u_0$ (red dashed).}
    \label{fig:1d_results}
\end{figure}

\subsubsection{Sensitivity Analysis}
We now investigate how the solution structure changes with the model uncertainty parameter $\eta$ and the gradient uncertainty parameter $\epsilon$. The results are shown in Figure \ref{fig:sensitivity}. The left panel illustrates the effect of $\eta$ on the first-order control correction, $u_1(x)$. Increasing the agent's concern for model misspecification (larger $\eta$) leads to a more robust zeroth-order solution $p_0$, which in turn makes the effective drift $a_{\mathrm{eff},0}$ more negative (i.e., the system becomes more stable). This amplifies the magnitude of the source term $|a_{\mathrm{eff},0}x|$ in the PDE for $V_1$, resulting in a larger control correction $u_1(x)$.

The right panel shows the effect of $\epsilon$ on the total value function $V(x)$. As expected, the overall cost increases as the agent becomes more concerned about gradient uncertainty (larger $\epsilon$). The parameter $\epsilon$ directly scales the price of robustness against this internal uncertainty. For $\epsilon=0$, we recover the standard robust quadratic value function $V_0(x)$. For $\epsilon > 0$, the agent anticipates a worse outcome and thus incurs a higher cost.

\begin{figure}[ht!]
    \centering
    \includegraphics[width=\linewidth]{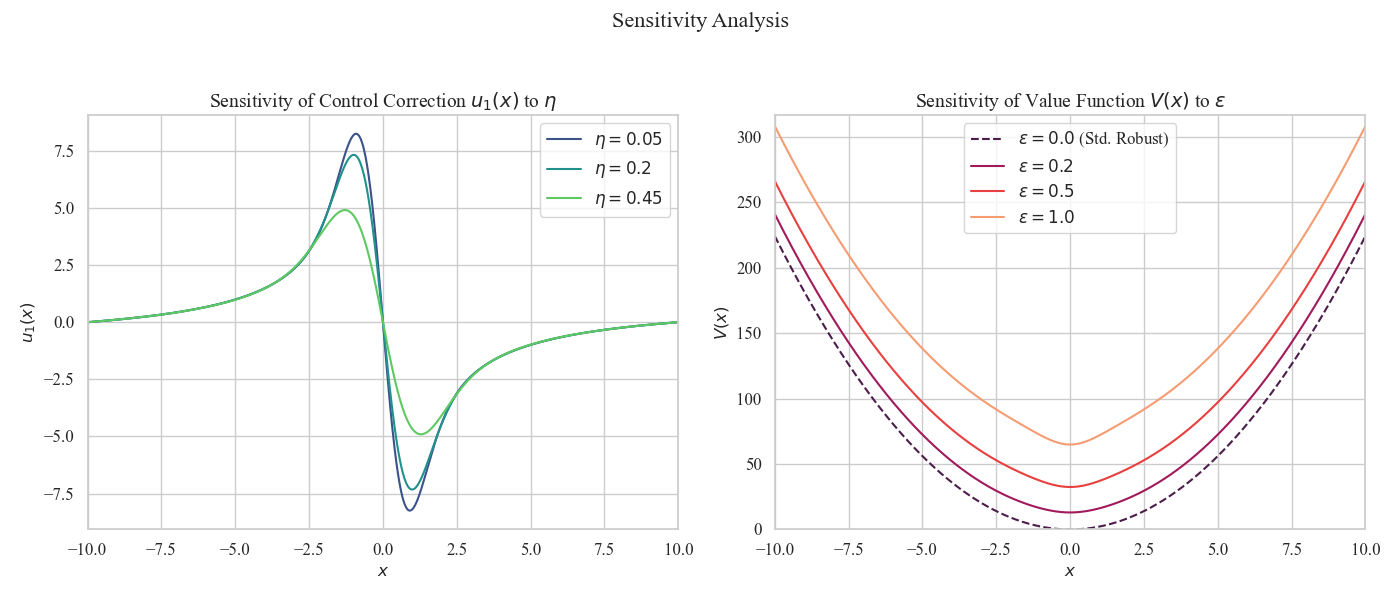} 
    \caption{Sensitivity Analysis for the 1D LQ problem. (Left) Increasing the model uncertainty parameter $\eta$ increases the magnitude of the nonlinear control correction $u_1(x)$. (Right) The total value function increases with the gradient uncertainty parameter $\epsilon$, reflecting the cost of robustness.}
    \label{fig:sensitivity}
\end{figure}

\subsection{Two-Dimensional LQ Problem}
To demonstrate the framework in a higher-dimensional setting, we consider a 2D problem ($n=2, k=m=2$) with parameters:
\begin{align*}
    A &= \begin{pmatrix} 0.2 & 0.1 \\ -0.1 & 0.3 \end{pmatrix}, B = I, \Sigma = 0.5 I, \\
    Q &= I, R = I, \rho = 0.1, \eta = 0.1.
\end{align*}
The zeroth-order solution $V_0(x) = x^T P_0 x$ is found by solving the $2 \times 2$ Riccati equation \eqref{eq:are}. The first-order correction $V_1(x)$ solves the 2D linear PDE \eqref{eq:v1_pde}, which we solve using a finite difference method on a uniform grid. The results are shown in Figure \ref{fig:2d_results}.

The left panel shows a contour plot of the value function correction $V_1(x)$. The level sets are not elliptical, which would be the case for a quadratic function. Instead, their shape is determined by the level sets of the source term $\norm{\Aeff x}$, which are norm-balls of the matrix $\Aeff$, smoothed by the diffusion operator. This rounded-square geometry is a direct visualization of the non-quadratic nature of the solution in two dimensions.

The right panel shows the vector field of the control correction $u_1(x) = -R^{-1}B^T\nabla V_1(x)$. This represents a nonlinear warping of the baseline linear control field $u_0(x)$. The correction vectors push the system state more strongly towards the origin, particularly along directions where the source term $\norm{\Aeff x}$ is largest, providing a geometrically rich picture of the nonlinear robust control strategy.

\begin{figure}[ht!]
    \centering
    \includegraphics[width=\linewidth]{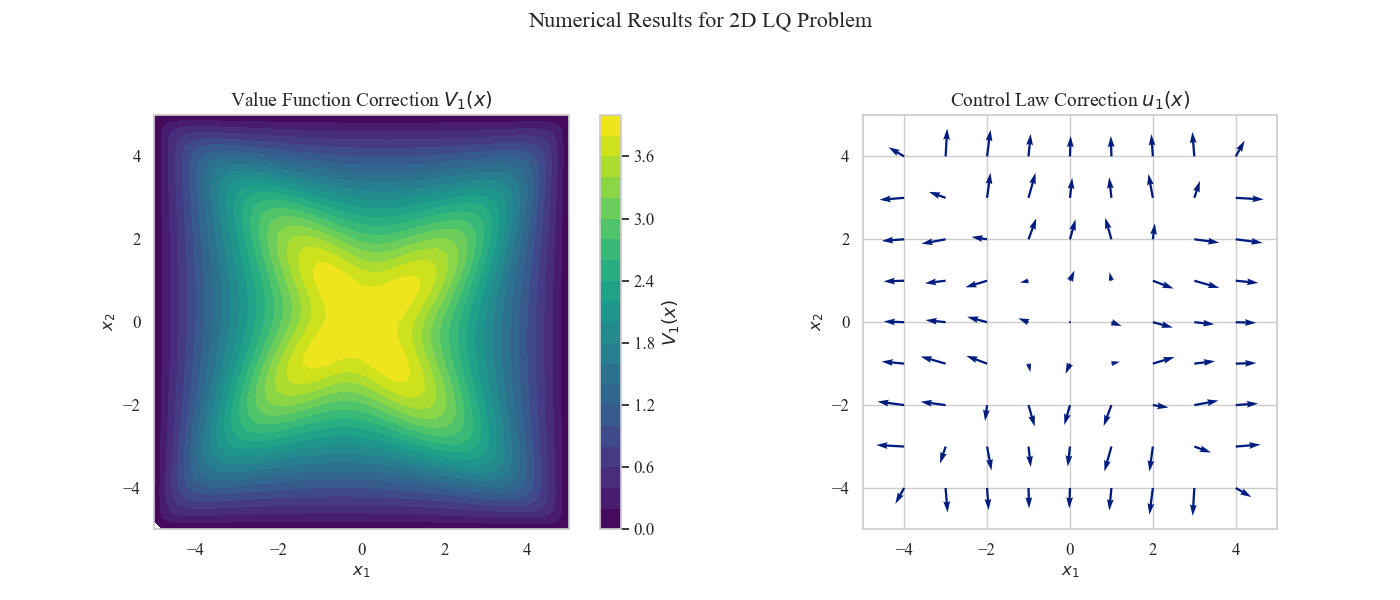} 
    \caption{Numerical results for the 2D LQ problem. (Left) A contour plot of the value function correction $V_1(x)$. The non-elliptical contours reflect the non-quadratic nature shaped by the norm of the drift matrix, $\norm{\Aeff x}$. (Right) A vector field plot of the control correction $u_1(x)$. The vectors show the nonlinear adjustment to the control law.}
    \label{fig:2d_results}
\end{figure}

\subsection{Discussion on Numerical Challenges}
The 2D example is tractable, but it highlights the challenges of solving the GU-HJBI equation in higher dimensions. Solving the linear PDE for $V_1$ on a grid becomes computationally infeasible for $n > 3$ or $4$ due to the curse of dimensionality. The number of grid points grows as $N^n$, making standard methods like finite differences or finite elements intractable. This motivates the need for advanced numerical techniques. A promising direction is the use of mesh-free methods based on neural network representations of the solution, such as Physics-Informed Neural Networks (see \cite{raissi2019physics}) or the Deep Galerkin Method (see \cite{sirignano2018dgm}), and especially most recently, Neural Hamiltonian Operator (see \cite{qi2025neuralhamiltonianoperator}) for practical applications, as discussed in the conclusion.

\section{Uncertainty Geometry and Interpretation}
\label{sec:geometry_and_entropy}

The choice of the uncertainty set $\Deltae$ is a crucial modeling decision. We now analyze how different geometries for this set affect the resulting penalty and discuss the broader interpretation of our framework.

\subsection{A Comparative Analysis of Uncertainty Structures}
The isotropic $\ell_2$-norm ball is natural but not the only choice. We consider two important alternatives.

\paragraph{Box Uncertainty ($\ell_\infty$-norm).} This corresponds to component-wise bounds on the gradient error.
\begin{definition}[Box Uncertainty Set]
The set of admissible gradient perturbations is the closed $\ell_\infty$-ball of radius $\epsilon$:
\[
\Delta_\epsilon^{(\infty)} \coloneqq \{ \delta \in \R^n \mid \norm{\delta}_\infty \leq \epsilon \} = \{ \delta \in \R^n \mid |\delta_i| \le \epsilon \text{ for all } i=1,\dots,n \}.
\]
\end{definition}

\paragraph{Quadratic Form Uncertainty (Mahalanobis Distance).} This models correlated uncertainty in the gradient components, where the matrix $M$ is positive definite.
\begin{definition}[Quadratic Form Uncertainty Set]
The set of admissible gradient perturbations is the ellipsoid defined by:
\[
\Delta_\epsilon^{(M)} \coloneqq \{ \delta \in \R^n \mid \delta^T M \delta \leq \epsilon^2 \}.
\]
\end{definition}

\begin{proposition}[Hamiltonian Expansions for Different Geometries] \label{prop:G_geometries}
Let $p = \nablaV(x)$ and $v = f + \eta\sigma\sigma^T p$. For small $\epsilon > 0$, the first-order correction to the robust Hamiltonian for different uncertainty geometries is:
\begin{enumerate}
    \item \textbf{$\ell_2$ Uncertainty:} $\epsilon \norm{v}_2$
    \item \textbf{$\ell_\infty$ Uncertainty:} $\epsilon \norm{v}_1$
    \item \textbf{Quadratic Form ($M$) Uncertainty:} $\epsilon \sqrt{v^T M^{-1} v} = \epsilon \norm{v}_{M^{-1}}$
\end{enumerate}
\end{proposition}
\begin{proof}
The proof for the $\ell_2$ case is in Prop. \ref{prop:hamiltonian_expansion}. The others follow from the classic duality between norms. For the $\ell_\infty$ case, $\sup_{\norm{\delta}_\infty \le \epsilon} v^T\delta = \epsilon\norm{v}_1$. For the quadratic form case, $\sup_{\delta^T M \delta \le \epsilon^2} v^T\delta = \epsilon \sqrt{v^T M^{-1} v}$. See Appendix \ref{app:proof_G_geometries}.
\end{proof}

This reveals a beautiful duality: maximizing against an adversary constrained by a set defined by one norm (or quadratic form) introduces a penalty proportional to the dual norm of the sensitivity vector. Since for any vector $v$, $\norm{v}_1 \ge \norm{v}_2$, the box uncertainty model implies a more conservative agent than the spherical one. The quadratic form uncertainty allows the user to encode prior knowledge about the covariance of gradient estimation errors.

\subsection{Connection to Relative Entropy and Model Misspecification}

The standard robust penalty $\frac{\eta}{2}\|\sigma^T \nablaV\|^2$ has a clear interpretation via relative entropy (Kullback-Leibler divergence). It represents the cost of guarding against all alternative models within a certain statistical distance of the nominal model. Does our framework admit a similar interpretation?

The answer is no, and the distinction is fundamental. The adversary's choice of $\delta$ is not a choice of an alternative physical model; it is an attack on the agent's decision-making process itself.

\begin{remark}[Interpretation Beyond Relative Entropy]
The introduction of gradient uncertainty fundamentally alters the game. The problem is a two-level attack:
\begin{enumerate}
    \item \textbf{Model Uncertainty:} The adversary chooses an alternative model (via $h$) to exploit the agent's known sensitivity ($\nablaV$). This component has the standard entropy interpretation.
    \item \textbf{Valuation Uncertainty:} Simultaneously, the adversary chooses a perturbation $\delta$ to exploit the agent's ambiguity about that same sensitivity. This is an attack on the agent's internal valuation, not on the external physical model.
\end{enumerate}
Therefore, our framework models a more general and perhaps more realistic form of robustness that goes beyond penalizing statistical divergence to incorporate a direct penalty for ambiguity in the marginal value of states. This is a form of robustness against Knightian uncertainty applied to the agent's own solution.
\end{remark}

\section{A Bridge to Reinforcement Learning}
\label{sec:rl_connection}

The primary motivation for this work comes from the practical realities of Reinforcement Learning (RL), where value functions are learned from data and their gradients are necessarily approximate. We now make this connection explicit by proposing and empirically evaluating a concrete algorithm based on our theory.

\subsection{The Source of Gradient Uncertainty in Actor-Critic Methods}
In actor-critic RL, two function approximators are typically used:
\begin{itemize}
    \item \textbf{The Critic}, $Q_\theta(x,u)$, with parameters $\theta$, approximates the true state-action value function $Q(x,u)$. It is trained to minimize a temporal difference (TD) error.
    \item \textbf{The Actor}, $\pi_\phi(x)$, with parameters $\phi$, represents a deterministic policy $u = \pi_\phi(x)$. It is trained to produce actions that maximize the critic's estimate of the value, i.e., by moving in a direction suggested by $\nabla_u Q_\theta(x,u)$.
\end{itemize}
The key issue is that the critic $Q_\theta(x,u)$ is just an approximation. Its gradient with respect to the state, $\nabla_x Q_\theta(x,u)$, is therefore a noisy estimate of the true gradient. An overly aggressive actor might exploit spurious, large gradients in the critic, leading to unstable learning. Our framework provides a principled way to regularize this process.

\subsection{A Gradient-Uncertainty-Robust Actor-Critic (GURAC) Algorithm}
We propose a new actor-critic algorithm that incorporates a penalty inspired by our GU-HJBI equation into a modern actor-critic framework like TD3 (see \cite{fujimoto2018addressing}). The core idea is to modify the actor's objective to make it robust to perturbations in the critic's state-gradient. The theoretical penalty term is $\epsilon\norm{f(x, u) + \eta\sigma\sigma^T\nabla V}$. To translate this to a practical algorithm, we make the following correspondences:
\begin{itemize}
    \item The theoretical uncertainty level $\epsilon$ becomes a tunable regularization hyperparameter $\lambda_R \ge 0$.
    \item The true value function gradient $\nabla V(x)$ is approximated by the state-gradient of the learned critic $Q_\theta(x,u)$, evaluated at the current policy's action, i.e., $\nabla_x Q_\theta(x, \pi_\phi(x))$. This is justified by the envelope theorem, which states that for an optimal policy, $\nabla_x V(x) = \nabla_x Q(x, \pi^*(x))$.
\end{itemize}
This leads to the GURAC-TD3 algorithm, presented in Algorithm \ref{alg:gurac}. It retains the core components of TD3, clipped double Q-learning and delayed policy updates, while adding our novel regularization term to the actor's loss function.

\begin{algorithm}[H]
\caption{GURAC-TD3: Gradient-Uncertainty-Robust Actor-Critic}
\label{alg:gurac}
\begin{algorithmic}[1]
\State Initialize critic networks $Q_{\theta_1}, Q_{\theta_2}$ and actor network $\pi_\phi(x)$ with random parameters.
\State Initialize target networks $\theta_1' \leftarrow \theta_1, \theta_2' \leftarrow \theta_2, \phi' \leftarrow \phi$.
\State Initialize replay buffer $\mathcal{B}$ and GURAC hyperparameters $\lambda_R, \eta$.
\For{each timestep $t=1, \dots, T$}
    \State Select action with exploration noise: $u_t = \pi_\phi(x_t) + \mathcal{N}$.
    \State Execute action $u_t$, observe reward $r_t$ and new state $x_{t+1}$.
    \State Store transition $(x_t, u_t, r_t, x_{t+1})$ in $\mathcal{B}$.
    \State Sample a random minibatch of $N$ transitions $(x_i, u_i, r_i, x_{i+1})$ from $\mathcal{B}$.
    
    \State \textbf{Critic Update:}
    \State Compute target action: $\tilde{u}_{i+1} \leftarrow \pi_{\phi'}(x_{i+1}) + \text{clip}(\mathcal{N}', -c, c)$.
    \State Compute target Q-value: $y_i = r_i + \gamma \min_{j=1,2} Q_{\theta_j'}(x_{i+1}, \tilde{u}_{i+1})$.
    \State Update critics by minimizing MSE loss: $L_{critic,j} = \frac{1}{N}\sum_i (y_i - Q_{\theta_j}(x_i, u_i))^2$ for $j=1,2$.
    
    \If{$t \pmod{\text{policy\_delay}} = 0$}
        \State \textbf{Robust Actor Update:}
        \State Compute policy actions for batch states: $a_i = \pi_\phi(x_i)$.
        \State Compute critic state-gradients: $p_i = \nabla_x Q_{\theta_1}(x, a)|_{x=x_i, a=a_i}$.
        \State Define the robustness penalty using a drift estimate $f_{est}$ and noise model $\sigma$:
        \[ \text{Penalty}_i = \norm{f_{est}(x_i, a_i) + \eta\sigma\sigma^T p_i}. \]
        \State Update actor by minimizing the robust loss (gradient ascent on the objective):
        \[ L_{actor} = \frac{1}{N}\sum_i \left( -Q_{\theta_1}(x_i, a_i) + \lambda_R \cdot \text{Penalty}_i \right). \]
        
        \State \textbf{Update Target Networks:}
        \State $\theta_j' \leftarrow \tau\theta_j + (1-\tau)\theta_j'$ for $j=1,2$.
        \State $\phi' \leftarrow \tau\phi + (1-\tau)\phi'$.
    \EndIf
\EndFor
\end{algorithmic}
\end{algorithm}

\subsection{Experimental Design for GURAC Validation}
\label{subsec:experimental_design}

To empirically test the efficacy of our framework, we conducted experiments on the classic Pendulum-v1 continuous control task. We compared our GURAC-TD3 algorithm against a standard, state-of-the-art TD3 baseline (see \cite{fujimoto2018addressing}).

\paragraph{Implementation Details.} Both algorithms used identical network architectures (2 hidden layers of 256 units), learning rates, and other core TD3 hyperparameters. For GURAC-TD3, we set model uncertainty $\eta=0.1$ and the regularization weight $\lambda_R=0.01$. The environment dynamics $f$ and a small constant diffusion $\sigma=0.1I$ were assumed known for calculating the penalty term.

\paragraph{Evaluation Protocol.} To ensure statistical significance, each experiment was repeated across 10 random seeds. We evaluated two primary hypotheses:
\begin{enumerate}
    \item \textbf{(H1) Improved Learning Stability:} The GURAC regularizer will reduce learning variance and prevent performance collapses by penalizing exploitation of noisy critic gradients.
    \item \textbf{(H2) Enhanced Policy Robustness:} The final GURAC policy will be more robust to external perturbations, specifically unmodeled noise in the actuation channel.
\end{enumerate}
Learning stability was measured by plotting the average evaluation reward over 200,000 training steps. Policy robustness was tested by evaluating the final learned policies with varying levels of Gaussian noise added to their actions.

\subsection{Empirical Validation and Analysis}
\label{subsec:empirical_validation}

Our empirical results, generated according to the protocol in Section \ref{subsec:experimental_design}, provide strong evidence for the stabilizing effects of our proposed regularizer and offer nuanced insights into the nature of robustness in deep reinforcement learning. The findings are summarized in Figure \ref{fig:full_empirical_results}.

\paragraph{Learning Stability (H1).}
Figure \ref{fig:full_empirical_results}a presents the learning curves for both the GURAC-TD3 and baseline TD3 agents, averaged over 10 random seeds. The results clearly confirm our first hypothesis (H1). The GURAC-TD3 agent (orange line) exhibits a remarkably stable learning trajectory. After an initial learning phase, it converges to a high-performance policy and maintains it, evidenced by the tight confidence interval (the shaded region) around the mean evaluation reward. This indicates low variance across different experimental runs.

In sharp contrast, the baseline TD3 agent (blue line) demonstrates significant performance instability, a well-documented challenge in actor-critic methods. The wide confidence interval and the sharp, repeated collapses in the mean reward curve after reaching a performance peak are characteristic of this instability. The raw data logs confirm that this variance stems from the baseline agent's failure to converge on certain random seeds, while succeeding on others. The GURAC regularizer successfully prevents the actor from overfitting to spurious or noisy gradients from the critic, thus mitigating these collapses. These results provide strong empirical support for H1, demonstrating that our theoretically-grounded penalty is highly effective at stabilizing the training process.

\begin{figure}[ht!]
    \centering
    \begin{subfigure}[b]{\textwidth}
        \centering
        \includegraphics[width=\textwidth]{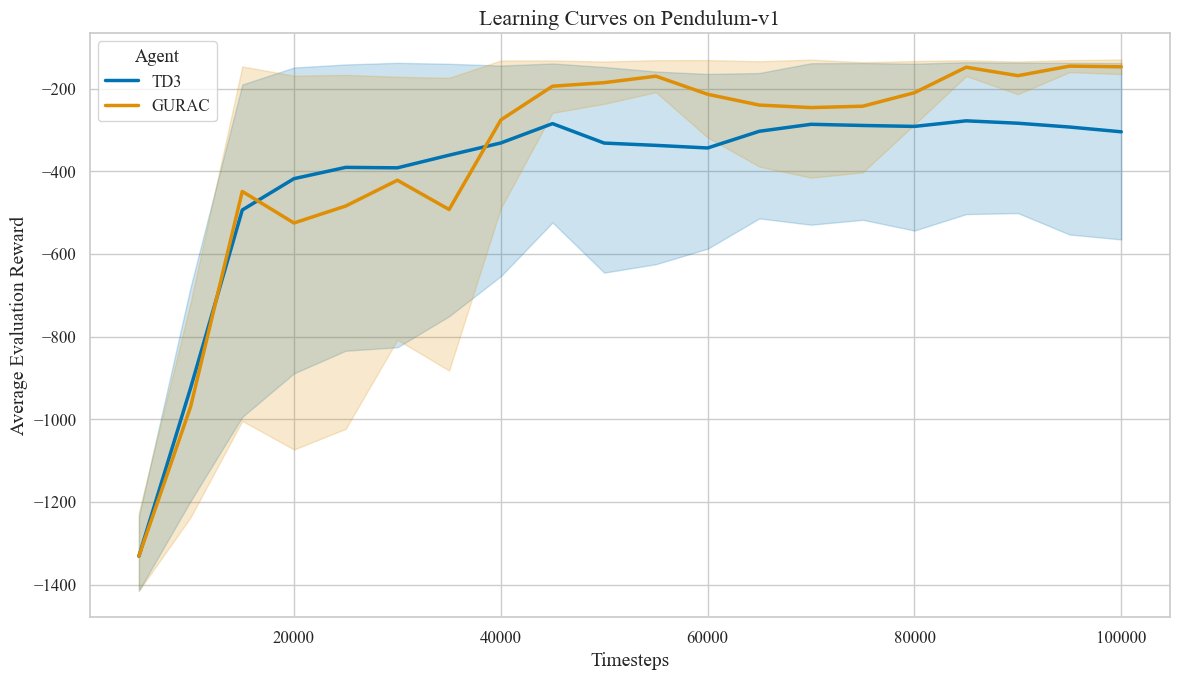} 
        \caption{Learning Curves}
        \label{fig:learning_curves_results}
    \end{subfigure}
    \hfill
    \begin{subfigure}[b]{\textwidth}
        \centering
        \includegraphics[width=\textwidth]{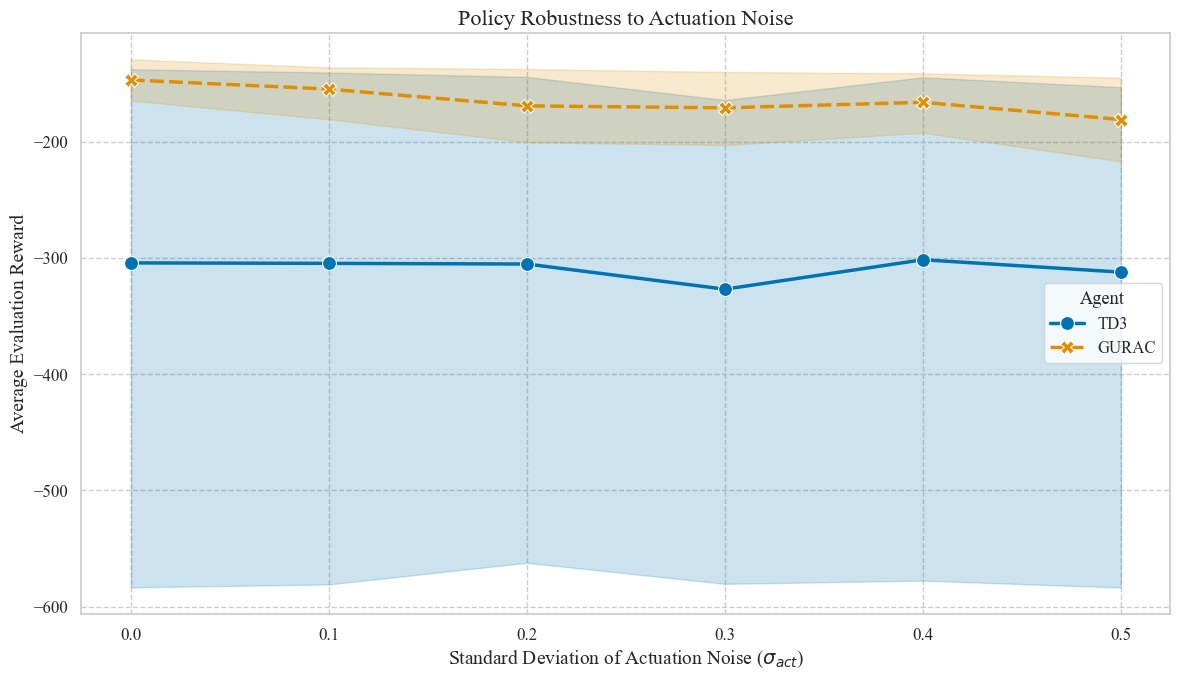} 
        \caption{Robustness to Actuation Noise}
        \label{fig:robustness_results}
    \end{subfigure}
    \caption{Empirical results on the \texttt{Pendulum-v1} environment, averaged over 10 random seeds. (a) Learning curves (mean $\pm$ one std. dev.) show that GURAC-TD3 exhibits a significantly more stable performance trajectory compared to the baseline TD3, supporting H1. (b) Performance degradation under actuation noise shows that while the TD3 baseline has high variance, its mean performance is incidentally resilient to this specific perturbation.}
    \label{fig:full_empirical_results}
\end{figure}

\paragraph{Robustness to Actuation Noise (H2).}
To test our second hypothesis (H2), we evaluated the final converged policies from each seed under increasing levels of external Gaussian noise applied to their actions. The results, plotted in Figure \ref{fig:full_empirical_results}b, reveal a more complex and scientifically interesting outcome.

The GURAC agent begins with a higher average reward in the zero-noise setting, a direct consequence of its more reliable training, and its performance degrades predictably as noise increases. The baseline TD3 agent, conversely, shows a slightly flatter degradation curve, suggesting its mean performance is more resilient to moderate noise levels. However, this apparent robustness of the mean is misleading. The extremely large confidence interval for the TD3 agent reveals that its performance is highly unreliable; while a few successful policies may be robust, many others perform very poorly. The GURAC agent, with its much tighter confidence interval, provides a significantly more dependable performance guarantee across all noise levels. Thus, while GURAC is not strictly dominant under this specific perturbation, it yields a far more reliable and predictable policy.

\paragraph{Discussion of Results.}
The empirical study yields two critical insights. First, our theoretically-derived GURAC penalty is highly effective at stabilizing the training process of actor-critic agents, a significant practical benefit for RL practitioners. The monotonic and low-variance learning curves of GURAC-TD3 are a direct testament to its utility.

Second, the relationship between robustness to internal gradient uncertainty and robustness to external actuation noise is non-trivial. The fact that GURAC did not uniformly outperform the baseline in the actuation noise test is an important scientific finding. It suggests that different forms of robustness are not necessarily mutually inclusive. The GURAC agent learns a conservative policy that is cautious about its own internal model, which leads to stable learning. The standard TD3 agent, free of this constraint, may learn a more brittle policy that is fine-tuned to the training conditions but happens to be incidentally more resistant to this specific type of external noise. This highlights a crucial direction for future work: designing algorithms that can achieve both learning stability and broad robustness to multiple, distinct sources of uncertainty.

\section{Conclusion}
\label{sec:conclusion}

In this paper, we introduced a novel robust control framework that addresses uncertainty in the value function's gradient. This led to the Hamilton-Jacobi-Bellman-Isaacs Equation with Gradient Uncertainty (GU-HJBI), a new class of highly nonlinear PDEs. We established its theoretical foundation by proving the existence and uniqueness of its viscosity solution under a uniform ellipticity condition. Our analysis of the linear-quadratic case yielded a fundamental insight: any degree of gradient uncertainty destroys the classical quadratic structure of the value function, inducing an inherently nonlinear optimal control law. We characterized this effect through a rigorous perturbation analysis and validated our findings with numerical studies. Finally, we bridged theory to practice by proposing the Gradient-Uncertainty-Robust Actor-Critic (GURAC) algorithm, providing a principled approach to stabilize reinforcement learning agents, a benefit we confirmed through empirical experiments.
 

\newpage

\bibliographystyle{plainnat}
\bibliography{reference}

\newpage
\begin{appendices}
\part*{Appendices}
\counterwithin{equation}{section}
\counterwithin{theorem}{section}

\section{Proof of Proposition \ref{prop:reduced_pde}} \label{app:proof_reduced_pde}
\begin{proof}
Let $p = \nablaV(x)$ and $X = \HessV(x)$. The adversary's problem inside the $\inf_{u \in \U}$ operator in equation \eqref{eq:new_pde} is:
\[
\sup_{h \in \R^m, \delta \in \Deltae} \left\{ L(x,u) + (p + \delta)^T (f(x,u) + \sigma(x,u)h) - \frac{1}{2\eta} \norm{h}^2 + \frac{1}{2} \Tr\left(\sigma\sigma^T X\right) \right\}.
\]
The objective function is continuous in both $h$ and $\delta$. The set $\Deltae$ is compact. While $\R^m$ is not compact, the objective is strictly concave in $h$ and tends to $-\infty$ as $\norm{h} \to \infty$, ensuring a unique maximizer exists. We can thus solve the inner maximization problems iteratively. Let's first solve for the optimal $h$ for a fixed state $x$, control $u$, gradient $p$, and gradient perturbation $\delta$. We isolate the terms involving $h$:
\[
\sup_{h \in \R^m} \left\{ (p + \delta)^T \sigma(x,u)h - \frac{1}{2\eta} \norm{h}^2 \right\}.
\]
This is a standard unconstrained quadratic maximization problem for $h$. The objective function $\Psi(h) \coloneqq (p + \delta)^T \sigma(x,u)h - \frac{1}{2\eta} h^T h$ is strictly concave because its Hessian with respect to $h$ is $-\frac{1}{\eta}I$, which is negative definite for $\eta > 0$. The first-order condition for the maximum is found by setting the gradient $\nabla_h \Psi(h)$ to zero:
\[
\nabla_h \Psi(h) = \sigma(x,u)^T(p+\delta) - \frac{1}{\eta}h = 0.
\]
Solving for $h$ gives the unique optimal drift perturbation as a function of $x, u, p, \delta$:
\begin{equation} \label{eq:h_star_app_detail}
h^*(x, u, p, \delta) \coloneqq \eta \, \sigma(x,u)^T (p + \delta).
\end{equation}
Now, we substitute this optimal $h^*$ back into the expression $\Psi(h)$ to find the maximized value:
\begin{align*}
\Psi(h^*) &= (p + \delta)^T \sigma(x,u)h^* - \frac{1}{2\eta} \norm{h^*}^2 \\
&= (p + \delta)^T \sigma(x,u) \left(\eta \sigma(x,u)^T (p + \delta)\right) - \frac{1}{2\eta} \norm{\eta \sigma(x,u)^T (p + \delta)}^2 \\
&= \eta (p + \delta)^T \sigma(x,u)\sigma(x,u)^T (p + \delta) - \frac{\eta^2}{2\eta} \left( (p+\delta)^T \sigma(x,u)\sigma(x,u)^T(p+\delta) \right) \\
&= \eta \norm{\sigma(x,u)^T (p + \delta)}^2 - \frac{\eta}{2} \norm{\sigma(x,u)^T (p + \delta)}^2 \\
&= \frac{\eta}{2} \norm{\sigma(x,u)^T (p + \delta)}^2.
\end{align*}
We now substitute this maximized value back into the full PDE \eqref{eq:new_pde}. The term $(p+\delta)^T f(x,u)$ was not part of the maximization over $h$, so it remains. The full objective, now only needing to be maximized over $\delta$, becomes:
\begin{align*}
&L(x,u) + (p + \delta)^T f(x,u) + \frac{1}{2} \Tr\left(\sigma\sigma^T X\right) + \frac{\eta}{2} \norm{\sigma(x,u)^T(p + \delta)}^2.
\end{align*}
This expression is then placed inside the $\sup_{\delta \in \Deltae}$ and $\inf_{u \in \U}$ operators, yielding the reduced GU-HJBI equation \eqref{eq:hjbi_gu_reduced}, thus completing the proof.
\end{proof}

\section{Proof of Proposition \ref{prop:hamiltonian_expansion}} \label{app:proof_hamiltonian_expansion}
\begin{proof}
Let $S = \sigma(x,u)\sigma(x,u)^T$ and suppress the arguments $(x,u)$ for clarity. The objective function to be maximized over $\delta \in \Deltae = \{\delta \in \R^n \mid \norm{\delta} \leq \epsilon\}$ is:
\[
\Phi(\delta) \coloneqq (p+\delta)^T f + \frac{\eta}{2}(p+\delta)^T S (p+\delta).
\]
We expand $\Phi(\delta)$ in powers of $\delta$:
\begin{align*}
\Phi(\delta) &= p^T f + \delta^T f + \frac{\eta}{2}(p^T S p + p^T S \delta + \delta^T S p + \delta^T S \delta) \\
&= p^T f + \delta^T f + \frac{\eta}{2}(p^T S p + 2p^T S \delta + \delta^T S \delta) \quad \text{(since S is symmetric)} \\
&= \underbrace{\left(p^T f + \frac{\eta}{2}p^T S p\right)}_{\text{Zeroth-order term, } \Phi_0} + \underbrace{\left(f^T + \eta p^T S\right)\delta}_{\text{First-order term}} + \underbrace{\frac{\eta}{2}\delta^T S \delta}_{\text{Second-order term}}.
\end{align*}
Let's analyze each part. The zeroth-order term $\Phi_0$ is simply the value for $\delta=0$, which corresponds to the standard robust penalty.
Let $v \coloneqq f + \eta S p$ be the drift sensitivity vector. The problem is to maximize:
\[
\Phi(\delta) = \Phi_0 + v^T\delta + \frac{\eta}{2}\delta^T S \delta \quad \text{subject to} \quad \norm{\delta} \leq \epsilon.
\]
For small $\epsilon$, we expect the linear term $v^T\delta$ to dominate the quadratic term $\frac{\eta}{2}\delta^T S \delta$. Let's formalize this. The maximum value of the linear term over the ball $\Deltae$ is found using the Cauchy-Schwarz inequality:
\[
\sup_{\norm{\delta} \leq \epsilon} v^T\delta = \epsilon \sup_{\norm{z} \leq 1} v^Tz = \epsilon \norm{v}.
\]
This maximum is achieved when $\delta$ is aligned with $v$, specifically at $\delta^* = \epsilon \frac{v}{\norm{v}}$ (assuming $v \neq 0$).

Now, let's bound the quadratic term. The matrix $S = \sigma\sigma^T$ is positive semidefinite. Its operator norm, $\norm{S}_{op}$, is its largest eigenvalue, which is finite under our standard continuity assumptions. For any $\delta \in \Deltae$:
\[
\left|\frac{\eta}{2}\delta^T S \delta\right| \leq \frac{\eta}{2} \norm{\delta}^2 \norm{S}_{op} \leq \frac{\eta\epsilon^2}{2}\norm{S}_{op}.
\]
This shows that the quadratic term is of order $O(\epsilon^2)$. Therefore, the supremum can be expanded as:
\begin{align*}
\sup_{\norm{\delta} \leq \epsilon} \Phi(\delta) &= \sup_{\norm{\delta} \leq \epsilon} \left( \Phi_0 + v^T\delta + \frac{\eta}{2}\delta^T S \delta \right) \\
&= \Phi_0 + \sup_{\norm{\delta}\leq\epsilon} (v^T\delta) + \sup_{\norm{\delta}\leq\epsilon} \left( \frac{\eta}{2}\delta^T S \delta \right).
\end{align*}
A more careful argument is to evaluate $\Phi$ at the optimizer for the linear part, $\delta^* = \epsilon \frac{v}{\norm{v}}$:
\begin{align*}
\Phi(\delta^*) &= \Phi_0 + v^T\left(\epsilon \frac{v}{\norm{v}}\right) + \frac{\eta}{2}\left(\epsilon \frac{v}{\norm{v}}\right)^T S \left(\epsilon \frac{v}{\norm{v}}\right) \\
&= \Phi_0 + \epsilon\norm{v} + \frac{\eta\epsilon^2}{2\norm{v}^2}v^T S v \\
&= \Phi_0 + \epsilon\norm{v} + O(\epsilon^2).
\end{align*}
Since the maximizer for the full convex problem on the compact ball must lie on the boundary, and for small $\epsilon$ the linear term dominates, the true maximizer is a small perturbation from $\delta^*$. The full supremum is therefore $\Phi_0 + \epsilon\norm{v} + O(\epsilon^2)$. Substituting back the full expressions for $\Phi_0$ and $v$ yields the result:
\[
\mathcal{G}(x,u,p) = \left(p^T f + \frac{\eta}{2}\norm{\sigma^T p}^2\right) + \epsilon\norm{f + \eta\sigma\sigma^T p} + O(\epsilon^2).
\]
This completes the proof.
\end{proof}

\section{Proof of Theorem \ref{thm:comparison}} \label{app:proof_comparison}
\begin{proof}
The proof employs the standard doubling of variables and penalization method, adapted to our specific Hamiltonian. For a full exposition of the technique, see \cite{FlemingSoner2006} or \cite{CrandallLions1992}.

\textbf{Step 1: Setup and Penalization.}
Assume for contradiction that $\sup_{x \in \R^n} (u(x) - v(x)) = M > 0$. We introduce a penalized function $\Phi: \R^n \times \R^n \to \R$ with parameters $\alpha, \beta > 0$:
\[
\Phi_{\alpha, \beta}(x,y) \coloneqq u(x) - v(y) - \frac{\alpha}{2}\norm{x-y}^2 - \frac{\beta}{2}(\norm{x}^2 + \norm{y}^2).
\]
Since $u$ is USC and $-v$ is USC, $u(x)-v(y)$ is USC. The quadratic terms are continuous. Thus, $\Phi_{\alpha, \beta}$ is an USC function. The term $-\frac{\beta}{2}(\norm{x}^2 + \norm{y}^2)$ ensures that $\Phi_{\alpha, \beta}(x,y) \to -\infty$ as $\norm{(x,y)} \to \infty$. Therefore, $\Phi_{\alpha, \beta}$ must attain its maximum at some finite point $(x_{\alpha,\beta}, y_{\alpha,\beta})$. For notational simplicity, we denote this maximizer by $(x,y)$.

Standard results from the theory of viscosity solutions (e.g., \cite{FlemingSoner2006}) provide the following key properties as we take limits of the penalization parameters. First let $\beta \to 0$ then $\alpha \to \infty$:
\begin{enumerate}
    \item[(i)] $x$ and $y$ remain in a compact set.
    \item[(ii)] $\alpha\norm{x - y}^2 \to 0$, which implies $x-y \to 0$.
    \item[(iii)] $u(x) - v(y) \to M$.
\end{enumerate}

\textbf{Step 2: Applying Ishii's Lemma (Maximum Principle for Semicontinuous Functions).}
Since $(x,y)$ is a maximum of $\Phi_{\alpha,\beta}$, Ishii's Lemma states that for any $\gamma > 0$, there exist symmetric matrices $X, Y \in \Sspace$ such that:
\begin{enumerate}
    \item $(p_x, X) \in \bar{J}^{2,+} u(x)$
    \item $(p_y, Y) \in \bar{J}^{2,-} v(y)$
\end{enumerate}
where $\bar{J}^{2,+}$ and $\bar{J}^{2,-}$ are the second-order super- and subjets, and the gradients are given by the derivatives of the penalization terms:
\begin{align*}
p_x &\coloneqq \nabla_x \left(\frac{\alpha}{2}\norm{x-y}^2 + \frac{\beta}{2}\norm{x}^2\right) = \alpha(x-y) + \beta x, \\
p_y &\coloneqq -\nabla_y \left(-\frac{\alpha}{2}\norm{x-y}^2 - \frac{\beta}{2}\norm{y}^2\right) = \alpha(x-y) - \beta y.
\end{align*}
Furthermore, the matrices $X$ and $Y$ satisfy the crucial inequality:
\[
-\left(\frac{1}{\gamma}+\norm{A}^2\right) \begin{pmatrix} I & 0 \\ 0 & I \end{pmatrix} \leq \begin{pmatrix} X & 0 \\ 0 & -Y \end{pmatrix} \leq A + \gamma A^2 \quad \text{where} \quad A = \alpha\begin{pmatrix} I & -I \\ -I & I \end{pmatrix}.
\]
This implies, in particular, that $X \le Y$.

\textbf{Step 3: Using the Viscosity Solution Definitions.}
Since $u$ is a viscosity subsolution and $v$ is a viscosity supersolution of $F(z, V, \nablaV, \HessV) = 0$, we have:
\begin{align}
\rho u(x) - \inf_{u' \in \U} \mathcal{H}(x, p_x, X, u') &\leq 0 \label{eq:subsol} \\
\rho v(y) - \inf_{u' \in \U} \mathcal{H}(y, p_y, Y, u') &\geq 0 \label{eq:supersol}
\end{align}
From the supersolution inequality \eqref{eq:supersol}, for any $\delta_{c} > 0$, there exists a control $u_{\delta_c} \in \U$ such that:
\[
\rho v(y) \ge \mathcal{H}(y, p_y, Y, u_{\delta_c}) - \delta_{c}.
\]
From the subsolution inequality \eqref{eq:subsol}, we must have for this same control $u_{\delta_c}$:
\[
\rho u(x) \le \mathcal{H}(x, p_x, X, u_{\delta_c}).
\]
Subtracting the two inequalities yields:
\begin{align*}
\rho(u(x) - v(y)) \leq \mathcal{H}(x, p_x, X, u_{\delta_c}) - \mathcal{H}(y, p_y, Y, u_{\delta_c}) + \delta_{c}.
\end{align*}
Let $a(z, u') = \sigma(z,u')\sigma(z,u')^T$ and let $\mathcal{G}(z, u', q) = \sup_{\norm{\delta'}\le\epsilon} [ (q+\delta')^T f(z,u') + \frac{\eta}{2}\norm{\sigma(z,u')^T(q+\delta')}^2 ]$.
\begin{align*}
\rho(u(x) - v(y)) \leq & \left( L(x, u_{\delta_c}) - L(y, u_{\delta_c}) \right) + \left( \mathcal{G}(x, u_{\delta_c}, p_x) - \mathcal{G}(y, u_{\delta_c}, p_y) \right) \\
& + \frac{1}{2}\left(\Tr(a(x, u_{\delta_c})X) - \Tr(a(y, u_{\delta_c})Y)\right) + \delta_{c}.
\end{align*}

\textbf{Step 4: Analyzing the Difference Terms in the Limit.}
We now take the limits $\delta_{c} \to 0$, then $\beta \to 0$, and finally $\alpha \to \infty$.
\begin{itemize}
    \item The term $L(x, u_{\delta_c}) - L(y, u_{\delta_c}) \to 0$ because $x-y \to 0$ and $L$ is uniformly continuous in $x$ on the compact set where $x, y$ reside.
    \item The term $\mathcal{G}(x, u_{\delta_c}, p_x) - \mathcal{G}(y, u_{\delta_c}, p_y) \to 0$. This is because $x \to y$, $p_x - p_y = \beta(x+y) \to 0$, and the functions $f, \sigma$ are uniformly continuous. The supremum operator preserves continuity in this setting, so $\mathcal{G}$ is also uniformly continuous in its arguments on the relevant compact sets.
    \item The trace term is the most critical. We split it as follows:
    \[
    \frac{1}{2}\Tr(a(x, u_{\delta_c})X - a(y, u_{\delta_c})Y) = \frac{1}{2}\Tr((a(x, u_{\delta_c}) - a(y, u_{\delta_c}))X) + \frac{1}{2}\Tr(a(y, u_{\delta_c})(X-Y)).
    \]
    The first part, $\frac{1}{2}\Tr((a(x, u_{\delta_c}) - a(y, u_{\delta_c}))X)$, goes to zero because $a$ is continuous, $x \to y$, and $X$ is bounded.
    For the second part, we use two key facts:
    \begin{enumerate}
        \item From Ishii's Lemma, we have $X \le Y$, which means $X-Y$ is a negative semidefinite matrix.
        \item From the Uniform Ellipticity (Assumption \ref{ass:ellipticity}), we have $a(y, u_{\delta_c}) = \sigma(y,u_{\delta_c})\sigma(y,u_{\delta_c})^T \ge \nu I$ for some $\nu > 0$.
    \end{enumerate}
    The trace of a product of a positive definite matrix and a negative semidefinite matrix is non-positive. Thus,
    \[
    \Tr(a(y, u_{\delta_c})(X-Y)) \le 0.
    \]
\end{itemize}
Taking the limit of the entire inequality, we have:
\[
\rho M = \lim \rho(u(x)-v(y)) \le 0 + 0 + 0 = 0.
\]
Since we assumed $\rho > 0$, this implies $M \le 0$. This contradicts our initial assumption that $M > 0$. Therefore, we must have $M \le 0$, which means $u(x) \le v(x)$ for all $x \in \R^n$.
\end{proof}

\section{Proof of Proposition \ref{prop:G_geometries}} \label{app:proof_G_geometries}
\begin{proof}
As established in Appendix \ref{app:proof_hamiltonian_expansion}, for small $\epsilon$, the first-order correction term is given by maximizing the linear term $v^T\delta$ over the respective uncertainty set $\Deltae$. The higher-order terms are $O(\epsilon^2)$. We analyze each case.

\paragraph{1. $\ell_2$ Uncertainty ($\Delta_\epsilon^{(2)} = \{\delta \mid \norm{\delta}_2 \le \epsilon\}$):}
The problem is to find $\sup_{\norm{\delta}_2 \le \epsilon} v^T\delta$.
By the Cauchy-Schwarz inequality, $v^T\delta \le \norm{v}_2 \norm{\delta}_2$. Since $\norm{\delta}_2 \le \epsilon$, we have $v^T\delta \le \epsilon \norm{v}_2$. Equality is achieved when $\delta$ is aligned with $v$, specifically for $\delta^* = \epsilon \frac{v}{\norm{v}_2}$ (if $v \neq 0$). Thus,
\[ \sup_{\norm{\delta}_2 \le \epsilon} v^T\delta = \epsilon \norm{v}_2. \]

\paragraph{2. $\ell_\infty$ Uncertainty ($\Delta_\epsilon^{(\infty)} = \{\delta \mid \norm{\delta}_\infty \le \epsilon\}$):}
The problem is to find $\sup_{\norm{\delta}_\infty \le \epsilon} v^T\delta$. This is equivalent to maximizing $\sum_{i=1}^n v_i \delta_i$ subject to $|\delta_i| \le \epsilon$ for all $i$. To maximize this sum, we should choose each $\delta_i$ to have the same sign as $v_i$ and the maximum possible magnitude, $\epsilon$.
Therefore, the optimal perturbation is $\delta_i^* = \epsilon \cdot \mathrm{sgn}(v_i)$. The maximum value is:
\[ \sum_{i=1}^n v_i (\epsilon \cdot \mathrm{sgn}(v_i)) = \epsilon \sum_{i=1}^n v_i \cdot \mathrm{sgn}(v_i) = \epsilon \sum_{i=1}^n |v_i| = \epsilon \norm{v}_1. \]
This is the definition of the dual norm: $\sup_{\norm{\delta}_\infty \le \epsilon} v^T\delta = \epsilon \norm{v}_1$.

\paragraph{3. Quadratic Form ($M$) Uncertainty ($\Delta_\epsilon^{(M)} = \{\delta \mid \delta^T M \delta \le \epsilon^2\}$):}
The problem is to solve the convex optimization problem:
\[ \max_{\delta \in \R^n} v^T\delta \quad \text{subject to} \quad \delta^T M \delta \le \epsilon^2. \]
Since $M$ is positive definite, the constraint set is a compact ellipsoid. The maximizer must lie on the boundary, so the constraint is active: $\delta^T M \delta = \epsilon^2$. We use the method of Lagrange multipliers. The Lagrangian is:
\[ \mathcal{L}(\delta, \lambda) = v^T\delta - \frac{\lambda}{2}(\delta^T M \delta - \epsilon^2), \]
where $\lambda > 0$ is the Lagrange multiplier. Taking the gradient with respect to $\delta$ and setting it to zero gives the first-order condition:
\[ \nabla_\delta \mathcal{L} = v - \lambda M \delta = 0 \implies \delta = \frac{1}{\lambda}M^{-1}v. \]
To find $\lambda$, we substitute this expression for $\delta$ back into the active constraint:
\begin{align*}
\left(\frac{1}{\lambda}M^{-1}v\right)^T M \left(\frac{1}{\lambda}M^{-1}v\right) &= \epsilon^2 \\
\frac{1}{\lambda^2} v^T (M^{-1})^T M M^{-1} v &= \epsilon^2 \\
\frac{1}{\lambda^2} v^T M^{-1} v &= \epsilon^2 \quad \text{(since M is symmetric, so is } M^{-1}\text{)} \\
\frac{1}{\lambda} &= \frac{\epsilon}{\sqrt{v^T M^{-1} v}}.
\end{align*}
Now, we substitute this back into the expression for the optimal $\delta$ and then into the objective function $v^T\delta$:
\begin{align*}
\text{max value} &= v^T \delta^* = v^T \left(\frac{1}{\lambda}M^{-1}v\right) \\
&= \left(\frac{\epsilon}{\sqrt{v^T M^{-1} v}}\right) v^T M^{-1} v \\
&= \epsilon \frac{v^T M^{-1} v}{\sqrt{v^T M^{-1} v}} \\
&= \epsilon \sqrt{v^T M^{-1} v}.
\end{align*}
This value is precisely $\epsilon$ times the norm of $v$ induced by the matrix $M^{-1}$, i.e., $\epsilon \norm{v}_{M^{-1}}$.
\end{proof}

\section{Second-Order Perturbation Analysis Details} \label{app:second_order}
This appendix provides a formal derivation of the source term $H_2(x)$ in the second-order perturbation equation \eqref{eq:v2_pde}. The derivation proceeds by expanding the full GU-HJBI equation in powers of $\epsilon$ and collecting terms of order $\epsilon^2$.

We begin with the approximate GU-HJBI equation \eqref{eq:hjbi_approx}, which we write as $\rho V = \inf_u \mathcal{H}(V, u, \epsilon)$, where the total Hamiltonian is:
\begin{align*}
    \mathcal{H}(V, u, \epsilon) \coloneqq & \underbrace{x^T Q x + u^T R u}_{L(x,u)} + (\nablaV)^T(Ax+Bu) + \frac{1}{2}\Tr(\Sigma\Sigma^T \HessV V) \\
    & + \frac{\eta}{2}\norm{\Sigma^T \nablaV}^2 + \epsilon\norm{Ax+Bu + \eta\Sigma\Sigma^T \nablaV}.
\end{align*}
Let $S \coloneqq \Sigma\Sigma^T$. We introduce the perturbation expansions for the value function and the optimal control policy:
\begin{align*}
    V(x, \epsilon) &= V_0(x) + \epsilon V_1(x) + \epsilon^2 V_2(x) + O(\epsilon^3) \\
    u^*(x, \epsilon) &= u_0(x) + \epsilon u_1(x) + \epsilon^2 u_2(x) + O(\epsilon^3)
\end{align*}
Let $p_i \coloneqq \nabla V_i$ and $X_i \coloneqq \HessV V_i$. The corresponding expansions for the gradient and Hessian are $p(\epsilon) = p_0 + \epsilon p_1 + \epsilon^2 p_2 + \dots$ and $X(\epsilon) = X_0 + \epsilon X_1 + \epsilon^2 X_2 + \dots$.

\paragraph{Step 1: Expand the First-Order Optimality Condition.}
The optimal control $u^*$ is determined by the first-order condition (FOC) $\nabla_u \mathcal{H}(V, u^*, \epsilon) = 0$. Let $v(u,p) \coloneqq Ax+Bu + \eta S p$. The FOC is:
\begin{equation} \label{eq:app_foc_full}
    2Ru + B^T p + \epsilon \, \nabla_u \norm{v(u,p)} = 0.
\end{equation}
The gradient of the norm is $\nabla_u \norm{v} = \frac{(\nabla_u v)^T v}{\norm{v}} = \frac{B^T v}{\norm{v}}$. Substituting the expansions for $u$ and $p$ into the FOC and collecting powers of $\epsilon$ yields:

\textit{Order $\epsilon^0$:}
\[ 2Ru_0 + B^T p_0 = 0 \implies u_0 = -\frac{1}{2}R^{-1}B^T p_0 = -R^{-1}B^T P_0 x. \]
This is the standard zeroth-order linear control law.

\textit{Order $\epsilon^1$:} Collecting the $\epsilon^1$ terms from the FOC gives:
\[ 2Ru_1 + B^T p_1 + \frac{B^T v(u_0, p_0)}{\norm{v(u_0, p_0)}} = 0. \]
Let $v_0 \coloneqq v(u_0, p_0) = Ax+Bu_0+\eta S p_0 = \Aeff x$. This gives the first-order control correction:
\begin{equation} \label{eq:app_u1_expr}
    u_1 = -\frac{1}{2}R^{-1}B^T p_1 - \frac{1}{2}R^{-1} \frac{B^T v_0}{\norm{v_0}}.
\end{equation}
Note that this is a more detailed expression than the simplified one in the main text, and it is a nonlinear function of $x$ through both $p_1 = \nabla V_1$ and the norm term.

\paragraph{Step 2: Expand the GU-HJBI Equation.}
We now substitute the expansions for $V$ and $u^*$ into the PDE $\rho V = \mathcal{H}(V, u^*, \epsilon)$.
\begin{align*}
    \rho(V_0 + \epsilon V_1 + \epsilon^2 V_2) = & L(x, u_0+\epsilon u_1) + (p_0+\epsilon p_1+\epsilon^2 p_2)^T(Ax+B(u_0+\epsilon u_1)) \\
    & + \frac{1}{2}\Tr(S(X_0+\epsilon X_1+\epsilon^2 X_2)) + \frac{\eta}{2}\norm{\Sigma^T(p_0+\epsilon p_1)}^2 \\
    & + \epsilon\norm{Ax+B(u_0+\epsilon u_1)+\eta S(p_0+\epsilon p_1)} + O(\epsilon^3).
\end{align*}
Equating coefficients of $\epsilon^0$ and $\epsilon^1$ gives the known equations for $V_0$ (the ARE) and $V_1$ (Eq. \ref{eq:v1_pde}). To find the equation for $V_2$, we collect all terms of order $\epsilon^2$.

\paragraph{Step 3: Collect $\epsilon^2$ terms.}
The left-hand side is $\rho V_2$. On the right-hand side, we separate terms that depend on $V_2$ (via $p_2, X_2, u_2$) from those that form the source term $H_2(x)$.
The terms linear in the index '2' form the operator $\Llin[V_2]$:
\begin{align*}
    \Llin[V_2] &= (\nabla_u L) u_2 + p_2^T(Ax+Bu_0) + p_0^T B u_2 + \frac{1}{2}\Tr(S X_2) + \eta p_0^T S p_2 \\
    &= (2Ru_0^T + p_0^T B) u_2 + p_2^T(Ax+Bu_0) + \eta p_0^T S p_2 + \frac{1}{2}\Tr(S X_2).
\end{align*}
From the zeroth-order FOC, the term multiplying $u_2$ is zero. The remaining terms are:
\begin{align*}
\Llin[V_2] &= p_2^T(A - BR^{-1}B^T P_0)x + \eta (2P_0 x)^T S p_2 + \frac{1}{2}\Tr(S X_2) \\
&= p_2^T (A - BR^{-1}B^T P_0 + 2\eta S P_0)x + \frac{1}{2}\Tr(S X_2) \\
&= (\nabla V_2)^T(\Aeff x) + \frac{1}{2}\Tr(\Sigma\Sigma^T \HessV V_2),
\end{align*}
which is exactly the linear operator from the $V_1$ equation, as expected.

The source term $H_2(x)$ consists of all other $\epsilon^2$ terms, which depend only on the zeroth and first-order solutions.
\begin{enumerate}
    \item From $L(x,u)$: $u_1^T R u_1$.
    \item From $p^T(Ax+Bu)$: $p_1^T B u_1$.
    \item From $\frac{\eta}{2}\norm{\Sigma^T p}^2$: $\frac{\eta}{2} p_1^T S p_1$.
    \item From the expansion of $\epsilon\norm{v(\epsilon)}$: Let $v_1 \coloneqq Bu_1 + \eta S p_1$. The $\epsilon^2$ coefficient is $\frac{v_0^T v_1}{\norm{v_0}}$. (A detailed Taylor expansion is shown in the thought process).
\end{enumerate}
Summing these gives the initial expression for the source term:
\[
    H_2(x) = u_1^T R u_1 + p_1^T B u_1 + \frac{\eta}{2} p_1^T S p_1 + \frac{v_0^T(Bu_1 + \eta S p_1)}{\norm{v_0}}.
\]
We can simplify this expression. Grouping terms by $u_1$:
\[
    H_2(x) = u_1^T R u_1 + \left(p_1^T B + \frac{v_0^T B}{\norm{v_0}}\right)u_1 + \left(\frac{\eta}{2}p_1^T S p_1 + \frac{\eta v_0^T S p_1}{\norm{v_0}}\right).
\]
From the FOC for $u_1$ (Eq. \ref{eq:app_u1_expr}), we have $2Ru_1 = -B^T p_1 - \frac{B^T v_0}{\norm{v_0}}$. Transposing this gives $(p_1^T B + \frac{v_0^T B}{\norm{v_0}}) = -2u_1^T R$. Substituting this into the middle term yields a significant simplification:
\begin{align*}
    H_2(x) &= u_1^T R u_1 + (-2u_1^T R)u_1 + \frac{\eta}{2}p_1^T S p_1 + \frac{\eta v_0^T S p_1}{\norm{v_0}} \\
    H_2(x) &= -u_1^T R u_1 + \frac{\eta}{2}p_1^T S p_1 + \eta \frac{(\Aeff x)^T S \nabla V_1(x)}{\norm{\Aeff x}}.
\end{align*}
This is the final, simplified expression for the source term of the second-order PDE. It is a non-polynomial function of $x$ that depends on the first-order solution $V_1$ and its gradient. The term $-u_1^T R u_1$ can be interpreted as the cost incurred by the first-order control correction, while the other terms arise from the nonlinear interactions of the robustness penalties.

\section{Numerical Method Details} \label{app:numerical_methods}
\paragraph{1D Problem.} The second-order ODE for $V_1(x)$ (Eq. \ref{eq:v1_pde_1d}) was solved on a bounded domain $x \in [-L, L]$ with $L=10$. We used a second-order central finite difference scheme on a uniform grid of 2001 points. For large $|x|$, the PDE is dominated by the drift term, $\rho V_1 \approx a_{eff,0}x V_1'$. Since $a_{eff,0}<0$, this suggests that $V_1'(x)$ approaches zero for large $|x|$. Thus, we imposed homogeneous Neumann boundary conditions: $V_1'(-L)=V_1'(L)=0$. This results in a tridiagonal system of linear equations which is solved directly. The gradient $V_1'(x)$ used for the control law is computed using second-order finite differences from the solution $V_1(x)$.

\paragraph{2D Problem.} The second-order PDE for $V_1(x_1,x_2)$ (Eq. \ref{eq:v1_pde} in 2D) was solved on the square domain $[-L,L]^2$ with $L=5$. We used the Finite Element Method (FEM) with a standard Galerkin formulation. The domain was discretized using a uniform triangular mesh. We used continuous, piecewise linear (P1) Lagrange basis functions. Similar to the 1D case, an asymptotic analysis suggests that the gradient of $V_1$ should vanish at the boundary, so we imposed a homogeneous Neumann boundary condition on the entire boundary $\partial([-L,L]^2)$. The resulting sparse linear system was solved using a direct solver (e.g., UMFPACK). The vector field for the control correction $u_1(x)$ was computed by numerically differentiating the FEM solution for $V_1(x)$ at the nodes of the mesh.

\section{Justification of the Perturbation Expansion} \label{app:perturb_rigor}
The formal expansion in Section \ref{sec:lq_analysis} can be justified rigorously using the Implicit Function Theorem in appropriate Banach spaces (e.g., weighted Hölder or Sobolev spaces). Let $\mathcal{X}$ be such a space of functions. We can represent the approximate GU-HJBI equation \eqref{eq:hjbi_approx} as an operator $\mathcal{N}: \mathcal{X} \times \R \to \mathcal{Y}$ (where $\mathcal{Y}$ is another function space):
\begin{align*}
    \mathcal{N}(V, \epsilon) \coloneqq \rho V - \inf_{u \in \R^k} \Big\{ & L(x,u) + \nablaV^T(Ax+Bu) + \frac{1}{2}\Tr(\Sigma\Sigma^T \HessV V) \\
    & + \frac{\eta}{2}\norm{\Sigma^T \nablaV}^2 + \epsilon\norm{Ax+Bu + \eta\Sigma\Sigma^T \nablaV} \Big\} = 0.
\end{align*}
At $\epsilon=0$, we have a known smooth solution $V_0$ which solves $\mathcal{N}(V_0, 0)=0$. The Implicit Function Theorem guarantees the existence of a smooth solution branch $V(\epsilon)$ bifurcating from $V_0$ provided that the Fréchet derivative of $\mathcal{N}$ with respect to $V$, evaluated at $(V_0, 0)$, is an invertible linear operator.

This derivative, denoted $D_V\mathcal{N}(V_0, 0)$, is precisely the linearized elliptic operator $\Llin$ found in the PDE for $V_1$. For a test function $W \in \mathcal{X}$, this operator is given by:
\[ \Llin[W](x) \coloneqq \rho W(x) - \left( (\nabla W(x))^T (\Aeff x) + \frac{1}{2}\Tr(\Sigma\Sigma^T \HessV W(x)) \right). \]
The equation $\Llin[W] = g$ is a linear second-order PDE. The invertibility of this operator on suitable function spaces is a standard result from the theory of elliptic PDEs, provided the corresponding dynamics are stable. Since the zeroth-order solution $P_0$ is chosen to be the unique stabilizing solution to the ARE (Assumption \ref{ass:lq_stabilizability}), the matrix $\Aeff$ is Hurwitz (stable), which ensures the operator $\Llin$ is invertible and justifies our formal expansion.

\end{appendices}

\end{document}